\documentclass{article}

\usepackage{arxiv}
\usepackage{lineno,hyperref}
\usepackage{array}

\usepackage{amsmath}
\usepackage{amssymb}  
\usepackage{amsthm}
\usepackage{hyperref}       
\usepackage{url}            
\usepackage{graphicx}
\usepackage{comment}
\usepackage{color}

\numberwithin{equation}{section}

\newcounter{counter} \numberwithin{counter}{section}
\newtheorem{theorem}[counter]{Theorem}

\newtheorem{definition}[counter]{Definition}
\newtheorem{assumption}[counter]{Assumption}
\newtheorem{lemma}[counter]{Lemma}
\newtheorem{corollary}[counter]{Corollary}
\newtheorem{remark}[counter]{Remark}

\newcommand{\ReLU}{\operatorname{ReLU}}
\newcommand{\Lip}{\operatorname{Lip}}

\newcommand{\req}[1]{Eq.\,(\ref{#1})}

\title{Variational Representations and Neural Network Estimation of R{\'e}nyi Divergences}

\author{
  Jeremiah Birrell\\
  Department of Mathematics and Statistics\\
  University of Massachusetts Amherst\\
  Amherst, MA 01003,  USA \\
  \texttt{birrell@math.umass.edu} \\
   \And
 Paul Dupuis\\
  Division of Applied Mathematics\\
  Brown University\\
  Providence, RI 02912, USA \\
  \texttt{dupuis@dam.brown.edu} \\
  \And
    Markos A. Katsoulakis\\
    Department of Mathematics and Statistics\\
  University of Massachusetts Amherst\\
  Amherst, MA 01003,  USA \\
  \texttt{markos@math.umass.edu} \\
\And
    Luc Rey-Bellet\\
    Department of Mathematics and Statistics\\
  University of Massachusetts Amherst\\
  Amherst, MA 01003,  USA \\
  \texttt{luc@math.umass.edu} \\
  \And
    Jie Wang\\
    Department of Mathematics and Statistics\\
  University of Massachusetts Amherst\\
  Amherst, MA 01003,  USA \\
  \texttt{wang@math.umass.edu} \\
}

\begin{document}

\maketitle

\begin{abstract}
We derive a new variational formula for the R{\'e}nyi family of divergences, $R_\alpha(Q\|P)$, between probability measures $Q$ and $P$. Our result generalizes the classical Donsker-Varadhan variational formula for  the Kullback-Leibler divergence. We further show that this R{\'e}nyi variational formula holds over a range of function spaces; this  leads to a formula for the optimizer under very weak assumptions and is also key in our development of a   consistency theory for R{\'e}nyi divergence estimators.  By applying this theory to neural-network estimators,   we show that if a neural network family satisfies one of several strengthened versions of the universal approximation property  then the corresponding R{\'e}nyi divergence estimator is  consistent. In contrast to density-estimator  based methods, our estimators involve only expectations under $Q$ and $P$ and hence are more effective in high dimensional systems.  We illustrate this via several numerical examples of neural network estimation in systems of up to 5000 dimensions.
\end{abstract}

\section{Introduction}
Information-theoretic divergences  are widely used to quantify the notion of `distance' between probability measures $Q$ and $P$; commonly used examples include the Kullback-Leibler divergence (i.e., KL-divergence or relative entropy), $f$-divergences, and R{\'e}nyi divergences.  The computation and estimation of divergences is important in many  applications, including independent component analysis \cite{hyvarinen2004independent}, medical image registration \cite{Maes}, feature selection \cite{1114861}, genomic clustering \cite{Butte}, the information bottleneck method \cite{info_bottleneck}, independence testing \cite{Kinney3354}, and in the analysis and design of generative adversarial networks (GANs) \cite{GAN,f_GAN,WGAN,10.5555/3295222.3295327, CumulantGAN:Pantazisetal}. 

Estimation of divergences from data is known to be a difficult problem \cite{doi:10.1162/089976603321780272,pmlr-v38-gao15}. Density-estimator  based methods such as those in \cite{Poczos:2011:NDE:3020548.3020618,NIPS2015_5911} are known to work best in low dimensions. However, recent work has shown that   variational representations of divergences can be used to construct  statistical estimators for the KL-divergence \cite{MINE_paper}, and more general $f$-divergences  \cite{Nguyen_Full_2010,Ruderman,birrell2020optimizing}, that scale better with dimension. The family of R\'{e}nyi divergences, first introduced in \cite{renyi1961measures}, provide  means of quantifying the discrepancy between two probability measures that are  especially sensitive to the relative tail behavior of the distributions. R{\'e}nyi divergences are used in variational inference \cite{li2016renyi},  uncertainty quantification for rare events \cite{10.1214/19-AAP1468},  and naturally arise in coding theory and hypothesis testing  (see \cite{van2014renyi} for further discussion and references). R{\'e}nyi divergences have several advantages over the commonly-used KL-divergence, including the ability to compare heavy-tailed distributions and certain non-absolutely continuous distributions. In addition, the estimation of KL-divergence can suffer from stability issues, due to the impact of rare events as well as high variance \cite{song2020understanding}, problems that we empirically find to be less pronounced for certain R{\'e}nyi divergences (see the example in Section \ref{sec:Renyi_HD_example} below). In this work we develop a new  variational characterization for the  family of R{\'e}nyi divergences, $R_\alpha(Q\|P)$, and study its use in statistical estimation. More specifically, we will prove 
\begin{align}\label{eq:renyi_preview}
&R_\alpha(Q\|P)=\sup_{g\in \Gamma}\!\left\{\frac{1}{\alpha-1}\log\!\left[\int e^{(\alpha-1)g}dQ\right]-\frac{1}{\alpha}\log\!\left[\int e^{\alpha g}dP\right]\right\}\,,\
\end{align}
where $\alpha\in\mathbb{R}$, $\alpha\neq 0,1$, and $\Gamma$ is an appropriate function space; see Theorem \ref{thm:gen_DV}  below.  \req{eq:renyi_preview} can be viewed as an extension   of the well-known Donsker-Varadhan variational formula  for the relative entropy \cite{DV1983, Dupuis_Ellis},
\begin{align}\label{eq:DV_var}
R(Q\|P)=\sup_{g\in \mathcal{M}_b({\Omega})}\left\{\int g dQ-\log\!\left[\int e^g dP\right]\right\},
\end{align}
where 
$\mathcal{M}_b({\Omega})$ denotes the set of bounded  measurable real-valued functions on $\Omega$. 
Note that \eqref{eq:renyi_preview}  generalizes \eqref{eq:DV_var} in two directions; we generalize both the divergence, $R(Q\|P)\to R_\alpha(Q\|P)$, and the function space,  $\mathcal{M}_b(\Omega)\to \Gamma$; allowed $\Gamma$'s are given in Theorem \ref{thm:gen_DV}, Corollary \ref{cor:g_star}, and Lemma \ref{lemma:Renyi_var_Phi} below. The flexibility in choosing $\Gamma$ allows us to derive a formula for the optimizer of \eqref{eq:renyi_preview}   under very weak assumptions (see Corollary \ref{cor:g_star}) and is also key in our development of consistent statistical estimators (see Lemma \ref{lemma:Renyi_var_Phi}  and Theorem \ref{thm:consistency}).

 The objective functional in the optimization problem \eqref{eq:renyi_preview} depends  on $Q$ and $P$ only through the expectation of certain functions of $g$. As a result, the objective functional can be estimated in a straightforward manner using only samples from $Q$ and $P$. This property makes \eqref{eq:renyi_preview} a powerful tool in the construction of statistical estimators for R{\'e}nyi divergences. In Section \ref{sec:consistency} we provide a general framework for proving consistency of   R{\'e}nyi divergence estimators that are  based on \eqref{eq:renyi_preview}.  In Section \ref{sec:NN_estimation} we apply this theory to show consistency of neural-network estimators.   Related methods were used to prove consistency of KL-divergence estimators in \cite{MINE_paper}, though under stronger assumptions. Here we contribute a set of new technical tools that allow for a consistency proof in important cases where the prior theory did not apply, specifically when the measures $Q$ and $P$ have non-compact support, are light-tailed, and for neural-network estimators with unbounded activation function, such as the widely-used ReLU activation. Our new method involves the use of the Tietze extension theorem and new strengthened versions of the universal approximation property (see Definitions \ref{def:Psi_bounded_approx} and \ref{def:Psi_bounded_L1_approx}) to vary the function  space, $\Gamma$, in the variational formula \eqref{eq:renyi_preview} (see Lemma \ref{lemma:Renyi_var_Phi}) and finally culminates in the consistency result, Theorem \ref{thm:consistency}.  Function spaces of neural-networks that satisfy the required assumptions are provided in Section \ref{sec:NN_estimation} and are discussed further in Section \ref{sec:NN_cons_proofs}.  Finally, in Section \ref{sec:numerical} we demonstrate the effectiveness of the R{\'e}nyi-divergence estimators   in numerical examples with  systems of up to 5000 dimensions.

\subsection{Related Work}
Our main result \eqref{eq:renyi_preview} can be viewed as a dual variational formula to the result in \cite{atar2015robust},  generalizing the duality between the Donsker-Varadhan and Gibbs variational principles.  An alternative variational formula for the R{\'e}nyi divergences, using an objective functional that is  a linear combination of relative entropies, can be found in Theorem 30 of \cite{van2014renyi} and also in Theorem 1 of \cite{8006657}.  As discussed above,  our result \eqref{eq:renyi_preview} is advantageous for the purpose of statistical estimation, as the objective functional is straightforward to estimate using only samples from $P$ and $Q$. This property was key in the use of \req{eq:DV_var} for the statistical estimation of KL-divergence and  applications to GANs in \cite{MINE_paper} and we will similarly take advantage of this property for R{\'e}nyi divergence estimation. In addition, our results on neural network estimation in Section \ref{sec:consistency}  provide theoretical underpinnings for   Cumulant GAN \cite{CumulantGAN:Pantazisetal}. Finally, we note that a variational formula for quantum R{\'e}nyi entropies was previously derived in \cite{Berta2017} and agrees with \eqref{eq:renyi_preview} in the commutative, discrete setting.

\section{Background on R{\'e}nyi Divergences}
The R{\'e}nyi divergence of order $\alpha\in(0,\infty)$, $\alpha\neq 1$, between two probability measures $Q$ and $P$ on a measurable space $(\Omega,\mathcal{M})$, denoted $R_\alpha(Q\|P)$, can be defined as follows: Let $\nu$ be a sigma-finite positive measure with  $dQ=qd\nu$ and $dP=pd\nu$. Then
\begin{align}\label{eq:Renyi_formula}
&R_\alpha(Q\|P)=\begin{cases} 
     \frac{1}{\alpha(\alpha-1)}\log\!\left[\int_{p>0} q^\alpha p^{1-\alpha} d\nu\right] &\begin{array}{l}\text{if }0<\alpha<1 \text{ or }\\
\alpha>1 \text{ and } Q\ll P \end{array}\\
       +\infty &\text{ if }\alpha>1\text{ and } Q\not\ll P.
         \end{cases}
\end{align}
Such a $\nu$ always exists (e.g., $\nu=Q+P$) and it can be shown that the definition \eqref{eq:Renyi_formula} does not depend on the choice of $\nu$. The $R_\alpha$ satisfy the following divergence property: $R_\alpha(Q\|P)\geq 0$ with equality if and only if $Q=P$.   In this sense, the R{\'e}nyi divergences provide a notion of `distance' between probability measures.  Note, however, that  R{\'e}nyi divergences are not symmetric, but rather they satisfy
\begin{align}\label{eq:R_neg_alpha}
R_\alpha(Q\|P)=R_{1-\alpha}(P\|Q),\,\,\,\,\,\alpha\in(0,1).
\end{align}
\req{eq:R_neg_alpha} is used to extend the definition of $R_\alpha(Q\|P)$ to $\alpha<0$.  R{\'e}nyi divergences are connected to the KL-divergence, $R(Q\|P)$, through the following limiting formulas:
\begin{align}
&\lim_{\alpha\to 1^-}\! R_\alpha(Q\|P)=R(Q\|P)
\end{align}
and if $R(Q\|P)=\infty$ or if $R_\beta(Q\|P)<\infty$ for some $\beta>1$ then
\begin{align}
&\lim_{\alpha\to 1^+} \!R_\alpha(Q\|P)=R(Q\|P).
\end{align}
See \cite{van2014renyi} for a detailed discussion of R{\'e}nyi divergences and proofs of these (and many other) properties. Note, however, that our definition of the R{\'e}nyi divergences is related to theirs by $D_\alpha(\cdot\|\cdot)=\alpha R_\alpha(\cdot\|\cdot)$.  Explicit formulas for the R{\'e}nyi divergence between members of many common parametric families can be found in \cite{GIL2013124}. R{\'e}nyi divergences are also connected with the family of $f$-divergences; see  \cite{1705001}.

\section{Variational Formula for the R{\'e}nyi Divergences}

The key result in the paper is the following  variational characterization of the R{\'e}nyi divergences, which generalizes the Donsker-Varadhan variational formula \eqref{eq:DV_var}.  The proof of  this theorem can  be found in Section \ref{app:gen_DV}.
\begin{theorem}[R{\'e}nyi-Donsker-Varadhan Variational Formula]\label{thm:gen_DV}
Let  $P$ and $Q$ be probability measures on $(\Omega,\mathcal{M})$ and $\alpha\in\mathbb{R}$, $\alpha\neq 0,1$. Then for any set of functions, $\Gamma$, with $\mathcal{M}_b(\Omega)\subset\Gamma\subset\mathcal{M}(\Omega)$ (where $\mathcal{M}(\Omega)$ denotes the set of all real-valued measurable functions on $\Omega$) we have
\begin{align}\label{eq:Renyi_var}
&R_\alpha(Q\|P)=\sup_{g\in \Gamma}\!\left\{\frac{1}{\alpha-1}\log\!\left[\int e^{(\alpha-1)g}dQ\right]-\frac{1}{\alpha}\log\!\left[\int e^{\alpha g}dP\right]\right\},
\end{align}
where we interpret $\infty-\infty\equiv-\infty$ and $-\infty+\infty\equiv-\infty$.  

If in addition $(\Omega,\mathcal{M})$ is a metric space with the Borel $\sigma$-algebra then \req{eq:Renyi_var} holds for all $\Gamma$ that satisfy $\Lip_b(\Omega)\subset\Gamma\subset\mathcal{M}(\Omega)$, where $\Lip_b(\Omega)$ denotes the space of bounded Lipschitz functions on $\Omega$ (we emphasize that the Lipschitz constant is allowed to take any finite value).
\end{theorem}
\begin{corollary}[Existence of an Optimizer]\label{cor:g_star}
Let $\alpha\in\mathbb{R}$,  $\alpha\neq 0,1$, and suppose $Q\ll P$, $dQ/dP>0$, $(dQ/dP)^\alpha\in L^1(P)$. Define  $g^*=\log(dQ/dP)$ and suppose $\Gamma$ is a function space that satisfies $g^*\in \Gamma\subset \mathcal{M}(\Omega)$.   Then \req{eq:Renyi_var} holds and the supremum is achieved at $g^*$.
\end{corollary}

The ability to vary the function space in \eqref{eq:Renyi_var} has several important consequences.
\begin{enumerate}
\item Taking $\Gamma=\mathcal{M}(\Omega)$, or some other appropriate set of unbounded functions, implies that one can use  unbounded activation functions (e.g., ReLU) in neural-network estimators of R{\'e}nyi divergences; see Section \ref{sec:NN_estimation}.
\item  For certain activation functions, taking $\Gamma=\Lip_b(\Omega)$  is key to proving the consistency of neural-network estimators based on \eqref{eq:Renyi_var}; see the third example in Section \ref{sec:NN_estimation} along with Section \ref{sec:NN_cons_proofs}.
\item  The ability to consider unbounded functions allows for existence of an optimizer under very general assumptions; see Corollary \ref{cor:g_star}. In some cases, the existence of an optimizer can be used to reduce the optimization to a finite dimensional problem; see Section \ref{sec:exp_family} below.
\end{enumerate} 

One can formally obtain the classical Donsker-Varadhan variational formula \eqref{eq:DV_var}  by  letting $\Gamma=\mathcal{M}_b(\Omega)$ and  taking $\alpha\to 1$ in \req{eq:Renyi_var}. Similarly, taking $\alpha\to 0$ and reindexing $g\to -g$ one obtains the Donsker-Varadhan variational formula for $R(P\|Q)$.  Rigorously, the extension of the Donsker-Varadhan variational formula to $\Gamma$ with $\mathcal{M}_b(\Omega)\subset\Gamma\subset\mathcal{M}(\Omega)$ follows from \req{eq:DV_var} together with Theorem 1 in  \cite{MINE_paper}. The generalization  to $\Lip_b(\Omega)\subset\Gamma\subset\mathcal{M}(\Omega)$ can be proven via the same method we use for R{\'e}nyi divergences (see \req{eq:Cb_reduction1} - \eqref{eq:Cb_reduction3} and the surrounding discussion). This is a new result to the best of our knowledge; we omit the details.

\begin{remark}
Note that the conventions regarding infinities in Theorem \ref{thm:gen_DV} are simply  convenient  short-hands that allow us to consider arbitrary unbounded functions.  If one wishes to avoid infinities in the objective functional then the optimization can  be restricted to 
\begin{align}
    \widetilde{\Gamma}\equiv\{g\in \Gamma:  \exp((\alpha-1)g)\in L^1(Q),\,\exp(\alpha g)\in L^1(P)\}
\end{align}
and the equality \eqref{eq:Renyi_var} will still hold.
\end{remark}

\subsection{ Variational Formula for the R{\'e}nyi Divergences: Exponential Families}\label{sec:exp_family}
 If $P$ and $Q$ are members of a  parametric family then, by using the formula for the optimizer $g^*=\log(dQ/dP)$, the function space $\Gamma$ can be further reduced  to a finite dimensional manifold of functions (here we assume the conditions from Corollary \ref{cor:g_star} that ensure the existence of $g^*$). In particular,  if $P=\mu_{\theta_p}$ and $Q=\mu_{\theta_q}$ are members of the same exponential family  $d\mu_\theta=h(x)e^{\kappa(\theta)\cdot T(x)-\beta(\theta)}\mu(dx),\,\,\,\theta\in\Theta$,  with $T:\Omega\to\mathbb{R}^k$  the vector of sufficient statistics and $\mu$ a $\sigma$-finite positive measure, then  the optimizer $g^*$ lies in the $(k+1)$-dimensional subspace of functions
 \begin{align}
   g_{(\Delta \kappa,\Delta\beta)}\equiv\Delta \kappa\cdot T-\Delta\beta\,,\,\,\,\, (\Delta \kappa,\Delta\beta)\in\mathbb{R}^{k+1}\,.
 \end{align}
 Computation of the R{\'e}nyi divergence therefore reduces to the following $k$-dimensional optimization problem (note that the R{\'e}nyi objective functional is  invariant under shifts, and so the $\Delta\beta$ terms cancel):
\begin{align}\label{eq:Renyi_manifold}
R_\alpha(Q\|P)=&\sup_{\Delta \kappa\in\mathbb{R}^{k} }\left\{\frac{1}{\alpha-1}\log\int e^{(\alpha-1)\Delta \kappa\cdot T}dQ-\frac{1}{\alpha}\log\int e^{\alpha \Delta \kappa\cdot T}dP\right\}\,.
\end{align}
Contrast this with an alternative parametric approach, wherein one estimates  $\theta_p$ and $\theta_q$ using maximum likelihood estimation and then uses the explicit formula  for the R{\'e}nyi divergence between members of an exponential family found in Chapter 2 in \cite{liese1987convex},
\begin{align}\label{eq:Renyi_exp_family_explicit}
R_\alpha(Q\|P)=\frac{1}{\alpha(\alpha-1)}\log\left(\frac{Z(\alpha\theta_q+(1-\alpha)\theta_p)}{Z(\theta_p)^{1-\alpha}Z(\theta_q)^\alpha}\right)\,,\,\,\,  \alpha>0\,,\,\,\alpha\neq 1\,,
\end{align}
where $Z(\theta)\equiv\exp(\beta(\theta))=\int h(x)e^{\kappa(\theta)\cdot T(x)}\mu(dx)$ is the partition function. Using \eqref{eq:Renyi_exp_family_explicit} to estimate the R{\'e}nyi divergence from data requires the solution of two optimization problems (one each to find maximum likelihood estimators for $\theta_q$ and $\theta_p$) and then the computation of three partition functions.  Even if one uses a more sophisticated method such as thermodynamic integration (see \cite{leli2010free}) to compute the partition functions in \eqref{eq:Renyi_exp_family_explicit}, there is still the challenge of generating data from $\mu_{\alpha\theta_q+(1-\alpha)\theta_p}$, which is required to address the  partition function in the numerator of \eqref{eq:Renyi_exp_family_explicit}.  These challenges are absent  when using  \eqref{eq:Renyi_manifold}, which only requires the solution of one optimization problem and can be estimated directly using samples from $Q$ and $P$; one does not need to generate samples from any auxiliary distribution. Therefore, we only expect \eqref{eq:Renyi_exp_family_explicit} to be preferable in simpler cases where the partition function can be computed analytically.  We illustrate the use of \eqref{eq:Renyi_manifold} to estimate R{\'e}nyi divergences in Section \ref{sec:exp_family_ex}.

\section{Statistical Estimation of R{\'e}nyi Divergences}\label{sec:consistency} We now discuss how the variational formula \eqref{eq:Renyi_var} can be used to construct statistical estimators for R{\'e}nyi divergences. The estimation of divergences in high dimensions is a difficult but important problem, e.g., for independence testing \cite{Kinney3354} and the development of GANs \cite{GAN,f_GAN,WGAN,10.5555/3295222.3295327, CumulantGAN:Pantazisetal}.  Density-estimator  based methods   for estimating divergences are known to be  effective primarily  in low-dimensions   (see \cite{Poczos:2011:NDE:3020548.3020618,NIPS2015_5911} as well as Figure 1 in \cite{MINE_paper} and further references therein). In contrast,  variational methods for  KL and $f$-divergences have proven   effective in a range of medium and high-dimensional systems \cite{MINE_paper,birrell2020optimizing}.  It should be noted that  high-dimensional problems still pose a considerable challenge in general; this is due in part to the problem of sampling rare events. However, existing Monte Carlo methods for sampling rare events (see, e.g., \cite{rubino2009rare,bucklew2013introduction,budhiraja2019analysis}) are  still applicable here.

The variational formula \eqref{eq:Renyi_var}  naturally suggests estimators of the form
\begin{align}\label{eq:renyi_est_def}
    \widehat{R}^{n,k}_\alpha(Q\|P)
    \equiv&\sup_{\phi\in\Phi_k}\left\{\frac{1}{\alpha-1}\log\!\left[\int e^{(\alpha-1)\phi}dQ_n\right]-\frac{1}{\alpha}\log\!\left[\int e^{\alpha \phi}dP_n\right]\right\}\,,
\end{align}
where $\Phi_k$ is an appropriate  family of functions (e.g., a neural network family) and $Q_n$, $P_n$ are the empirical measures constructed from $n$ independent samples from $Q$ and $P$ respectively. Note that there are two levels of approximation here: we approximate the measures $Q\approx Q_n$, $P\approx P_n$, and we approximate the function space $\Gamma\approx \Phi_k$, with the approximations becoming arbitrarily good (in the appropriate senses) as $n,k\to\infty$.  In Theorem \ref{thm:consistency}  below we will give a consistency result for \eqref{eq:renyi_est_def}; under appropriate assumptions we will show that for all $\delta>0$ there exists $K\in\mathbb{Z}^+$ such that for all $k\geq K$ we have
\begin{align}
\lim_{n\to\infty}\mathbb{P}\left( \left|R_\alpha(Q\|P)-\widehat{R}_\alpha^{n,k}(Q\|P)\right|\geq\delta\right)=0\,.
\end{align}

Theorem \ref{thm:gen_DV} implies that the $\Phi_k$ are allowed to contain unbounded functions, an important point for practical computations. In addition, note the objective functional in \req{eq:renyi_est_def} only involves the values of $\phi$ at the sample points; there is no need to estimate the likelihood ratio $dQ/dP$.   In contrast, estimators of the form \eqref{eq:renyi_est_def} perform well in high dimensions, as we demonstrate below in Section \ref{sec:numerical}.

\subsection{Neural Network Estimators For R{\'e}nyi Divergences}\label{sec:NN_estimation}
While we will provide a general consistency theory for the estimator \eqref{eq:renyi_est_def} in Section \ref{sec:gen_consistency},  we are primarily interested in neural-network estimators on $\Omega=\mathbb{R}^m$, i.e., where the $\Phi_k$ in \eqref{eq:renyi_est_def} are  neural network families. By a neural network family, we mean a collection of functions, $\phi:\mathbb{R}^m\to\mathbb{R}$ (here, $\mathbb{R}^m$ is called the input layer and $\mathbb{R}$ the output layer) that are constructed as follows: First  compose some number, $d$, of hidden layers of the form $\sigma_j\circ B_{j-1}$, where $B_{j-1}:\mathbb{R}^{m_{j-1}}\to\mathbb{R}^{m_{j}}$ is affine ($m_0\equiv m$) and $\sigma_j:\mathbb{R}^{m_j}\to\mathbb{R}^{m_j}$ is a (nonlinear) activation function.   Then finish by composing with a final affine map $B_d:\mathbb{R}^{m_d}\to\mathbb{R}$. Often, the   $\sigma_j$'s are defined by applying a nonlinear function  $\sigma:\mathbb{R}\to\mathbb{R}$ to each of the $m_j$ components; in such a case, we will call $\sigma$ the activation function. The parameters of the neural network consist of the (weight) matrices and  shift (i.e., bias) vectors from all  affine transformations used in the construction (for technical reasons, we will assume that the set of allowed weights and biases is closed). The number of hidden layers is called the depth of the network and the dimension of each layer is called its width.

As we will see in Theorem \ref{thm:consistency} below, consistency of the estimator \eqref{eq:renyi_est_def}  will rely on the ability of  $\Phi\equiv \cup_k\Phi_k$ to approximate $\Gamma=\Lip_b(\mathbb{R}^m)$ in the appropriate sense. Neural networks are well suited for this task, as they   satisfy various  versions of the   universal approximation property.  The two most common variants are:
\begin{enumerate}
    \item[a.]  For all $g\in C(\mathbb{R}^m)$, all $\epsilon>0$, and all compact $K\subset\mathbb{R}^m$  there exists $\phi\in\Phi$ such that \begin{align}\label{eq:univ_approx_K}
    \sup_{x\in K}|g(x)-\phi(x)|<\epsilon\,.
\end{align}
\item[b.]  Let  $p\in[1,\infty)$. For all $g\in L^p(\mathbb{R}^m)$ and all $\epsilon>0$   there exists $\phi\in\Phi$ such that \begin{align}\label{eq:univ_approx_L1}
   \int_{\mathbb{R}^m}|g(x)-\phi(x)|^pdx<\epsilon\,.
\end{align}
\end{enumerate}
For example,  under suitable assumptions the family of (shallow) arbitrary width neural networks  satisfies \eqref{eq:univ_approx_K} \cite{Cybenko1989,pinkus_1999}. Results for deep networks with bounded width are also known; see \cite{pmlr-v125-kidger20a} for \req{eq:univ_approx_K} and \cite{10.5555/3295222.3295371,park2021minimum} for \req{eq:univ_approx_L1}.  Here we will only work with neural networks consisting of continuous functions, i.e., those with continuous activation functions; this is true of most activation functions used in practice. 

We will prove that consistency of a neural-network estimator follows from one of several strengthened versions of the universal approximation property; we introduce these in Definitions \ref{def:Psi_bounded_approx} and \ref{def:Psi_bounded_L1_approx} below.  Before presenting  these details, we first give three classes of networks to which our consistency result (Theorem~\ref{thm:consistency}) will apply; proofs that all  required assumptions   are satisfied  can be found in Section \ref{sec:NN_cons_proofs}.
\begin{enumerate}
    \item Measures with compact support: Let $\Omega\subset\mathbb{R}^m$ be compact, $\Phi$ be a family of neural networks that satisfy the universal approximation property \eqref{eq:univ_approx_K}, and let $\Phi_k\subset\Phi$ be the set of networks with depth and width bounded by $k$ and with parameter values restricted to  $[-a_k,a_k]$, where $a_k\nearrow\infty$. Then  the estimator \eqref{eq:renyi_est_def} is consistent.

    \item Non-compact support, bounded Lipschitz activation functions: Let $\Omega=\mathbb{R}^m$ and $\Phi$ be the family of neural networks with 2 hidden layers, arbitrary width, and activation function $\sigma:\mathbb{R}\to\mathbb{R}$. Let $\Phi_k\subset\Phi$ be the set of width-$k$ networks with parameter values restricted to $[-a_k,a_k]$, where $a_k\nearrow\infty$ (this family of networks satisfies \eqref{eq:univ_approx_K}).  If the activation function, $\sigma$,  is bounded and there exists  $(c,d)\subset\mathbb{R}$ on which $\sigma$ is one-to-one and Lipschitz  then  the estimator \eqref{eq:renyi_est_def} is consistent.

    \item Non-compact support, unbounded Lipschitz activation functions: Let $p\in(1,\infty)$ and $\Omega=\mathbb{R}^m$. Let $Q$ and $P$ be probability measures on $\Omega$ with  finite moment generating functions everywhere and with densities $dQ/dx$ and $dP/dx$ that are bounded on compact sets. Let  $\Phi$ be the family of neural networks obtained by using either the ReLU activation function or the GroupSort activation with group size 2 (these satisfy variants of \eqref{eq:univ_approx_K} and \eqref{eq:univ_approx_L1}, see Theorem 1 in \cite{park2021minimum} and Theorem 3 in \cite{pmlr-v97-anil19a} respectively); note that these activations are unbounded, hence in this case it is critical that Theorem \ref{thm:gen_DV} applies to spaces of unbounded functions. Finally, let $\Phi_k\subset\Phi$ be the set of networks with depth and width bounded by $k$ and with parameter values restricted to  $[-a_k,a_k]$, where $a_k\nearrow\infty$.    Then  the estimator \eqref{eq:renyi_est_def} is consistent. For ReLU activations our proof shows that 3 hidden layers is sufficient.    
\end{enumerate}
Note that in all cases, the $\Phi_k$'s are an increasing family of neural networks  with parameter values restricted to an increasing family of compact sets. Similar boundedness assumptions on the network parameters were required    in \cite{MINE_paper}, which studied neural-network estimators for the KL-divergence.  Apart from generalizing to R{\'e}nyi divergences, the primary contributions of the current work are several new approximation results which enable us to consider $Q$ and $P$ with non-compact support as well as unbounded activation functions. In contrast, the consistency result for KL divergence in \cite{MINE_paper} only applies to compactly supported measures (in which case  boundedness of the activation is irrelevant).

\subsection{Consistency of the R{\'e}nyi Divergence Estimators} \label{sec:gen_consistency} Though we are primarily interested in neural-network estimators, we will present our consistency result in terms of abstract requirements on the approximation spaces $\Phi_k$. Intuitively, the basic requirement is that $\Phi\equiv\cup_k\Phi_k$ is `dense' in $\Lip_b(\Omega)$ in the appropriate sense.  More precisely, we will need a space of functions, $\Phi$, that satisfies one of the following strengthened/modified versions of the universal approximation properties from \req{eq:univ_approx_K} and \eqref{eq:univ_approx_L1}:
\begin{definition}\label{def:Psi_bounded_approx}
Let $\Omega$ be a metric space and $\Phi,\Psi\subset \mathcal{M}(\Omega)$. We say that $\Phi$ has the {\bf  $\Psi$-bounded $L^\infty$ approximation property} if the following two properties hold:
\begin{enumerate}
\item For all $\phi\in\Phi$ there exists $\psi\in\Psi$ with $|\phi|\leq\psi$.
\item For all $g\in \Lip_b(\Omega)$  there exists $\psi\in\Psi$ such that:
\begin{enumerate}
    \item $|g|\leq \psi$.
    \item For all   compact $K\subset\Omega$ and all  $\epsilon>0$ there exists $\phi\in\Phi$ with  $|\phi|\leq \psi$ and $\sup_{x\in K}|g(x)-\phi(x)|<\epsilon$.
    \end{enumerate}
\end{enumerate}
\end{definition}
\begin{definition}\label{def:Psi_bounded_L1_approx}
Let $\Omega$ be a metric space, $\mathcal{Q}$ be a collection of Borel probability measures on $\Omega$, and $\Phi,\Psi\subset \mathcal{M}(\Omega)$. Let $p\in[1,\infty)$. We say that $\Phi$ has the {\bf $\Psi$-bounded $L^p(\mathcal{Q})$ approximation property} if the following two properties hold:
\begin{enumerate}
\item For all $\phi\in\Phi$ there exists $\psi\in\Psi$ with $|\phi|\leq\psi$.
\item For all $g\in \Lip_b(\Omega)$  there exists $\psi\in\Psi$ such that:
\begin{enumerate}
    \item $|g|\leq \psi$.
    \item For all   compact $K\subset\Omega$ and all  $\epsilon>0$ there exists $\phi\in\Phi$ with $|\phi|\leq \psi$ and\\ $\sup_{\mu\in\mathcal{Q}}\left(\int_K |g-\phi|^pd\mu\right)^{1/p}<\epsilon$.
    \end{enumerate}
\end{enumerate}
\end{definition}
Intuitively, these definitions state that functions in $\Phi$ are able to approximate bounded Lipschitz functions on compact sets (in some norm), and with the   approximating functions being uniformly bounded on the whole space by some fixed function in $\Psi$.  For the neural network families 1 and 2 of Section \ref{sec:NN_estimation} we will let $\Psi$ be the set of positive constant functions and in case 3 we will let $\Psi=\{x\mapsto a\|x\|+b:a,b\geq 0\}$; see Section \ref{sec:NN_cons_proofs} for  details.

Under appropriate integrability assumptions on $\Psi$, the ability to approximate in either of the above manners allows one to restrict the optimization in \eqref{eq:Renyi_var} to $\Phi$, leading to the following result (the proof can be found in Section \ref{sec:consistency_proof}).
\begin{lemma}\label{lemma:Renyi_var_Phi}
Let $\Omega$ be a complete separable metric space, $Q,P$ be Borel probability measures on $\Omega$, $\alpha\in\mathbb{R}\setminus\{0,1\}$, and  $\Phi,\Psi\subset \mathcal{M}(\Omega)$. Suppose one of the following two collections of properties holds: 
\begin{enumerate}
    \item 
\begin{enumerate}
\item  $\Phi$ has the  $\Psi$-bounded  $L^\infty$ approximation property.
\item $e^{\pm(\alpha-1)\psi}\in L^1(Q)$ for all $\psi\in\Psi$.
\item $e^{\pm\alpha\psi} \in L^1(P)$ for all $\psi\in\Psi$.
\end{enumerate}

    \item There exist conjugate exponents $p,q\in(1,\infty)$ such that:
\begin{enumerate} 
\item   $\Phi$ has the  $\Psi$-bounded $L^p(\mathcal{Q})$ approximation property, where $\mathcal{Q}\equiv\{Q,P\}$.
\item $e^{\pm q(\alpha-1)\psi}\in L^1(Q)$ for all $\psi\in\Psi$.
\item $e^{\pm q\alpha\psi} \in L^1(P)$ for all $\psi\in\Psi$.
\end{enumerate}
\end{enumerate}

 Then
\begin{align}\label{eq:var_formula_Phi}
R_\alpha(Q\|P)= &\sup_{\phi\in\Phi}\left\{\frac{1}{\alpha-1}\log\int e^{(\alpha-1)\phi}dQ-\frac{1}{\alpha}\log\int e^{\alpha \phi}dP\right\}\,.
\end{align}
\end{lemma}

We will be able to prove consistency of the estimator \eqref{eq:renyi_est_def} when the approximation spaces, $\Phi_k$,  increase to a function space, $\Phi$,  that satisfies the assumptions of Lemma \ref{lemma:Renyi_var_Phi}. More specifically (and slightly more generally), we will work under the following set of assumptions.
\begin{assumption}\label{assump:Phi_k} 
Suppose we have $\Phi_k,\Psi\subset \mathcal{M}(\Omega)$ that satisfy the following:
\begin{enumerate}
\item 
\begin{align}\label{eq:k_lim_assump}
R_\alpha(Q\|P)= &\lim_{k\to\infty}\sup_{\phi\in\Phi_k}\left\{\frac{1}{\alpha-1}\log\int e^{(\alpha-1)\phi}dQ-\frac{1}{\alpha}\log\int e^{\alpha \phi}dP\right\}\,.
\end{align}
\item Each  $\Phi_k$ has the form
\begin{align}
\Phi_k=\{\phi_k(\cdot,\theta):\theta\in\Theta_k\}\,,
\end{align}
where $\phi_k:\Omega\times\Theta_k\to\mathbb{R}$ is continuous and $\Theta_k$ is a compact metric space.
\item For each $k$ there exists $\psi_k\in\Psi$ with $\sup_{\theta\in\Theta_k}|\phi_k(\cdot,\theta)|\leq \psi_k$.
\item $e^{\pm(\alpha-1)\psi}\in L^1(Q)$ for all $\psi\in\Psi$.
\item $e^{\pm\alpha\psi} \in L^1(P)$ for all $\psi\in\Psi$.
\end{enumerate}
\end{assumption}
Our primary means of  satisfying the condition \eqref{eq:k_lim_assump} is described in the following  lemma. 
\begin{lemma}\label{lemma:union_Phik}
Suppose $\Phi$ satisfies the assumptions of Lemma \ref{lemma:Renyi_var_Phi}.  Take subsets $\Phi_k\subset\Phi_{k+1}\subset \Phi$, $k\in\mathbb{Z}^+$, with $\cup_k\Phi_k=\Phi$. Then the equality \eqref{eq:k_lim_assump} holds.
\end{lemma}
We use  this lemma in the concrete examples in Section \ref{sec:NN_estimation} and the proofs in Section \ref{sec:NN_cons_proofs}. However, we will not directly use Lemma \ref{lemma:union_Phik} in the proof of the  consistency result, Theorem \ref{thm:consistency}; there we will work  under the more general  Assumption \ref{assump:Phi_k}. We now state our consistency result.
\begin{theorem}\label{thm:consistency}
Let $\alpha\in\mathbb{R}\setminus\{0,1\}$, $\Omega$ be a complete separable metric space, $P,Q$ be Borel probability measures on $\Omega$,  and  $X_i,Y_i$, $i\in\mathbb{Z}_+$ be $\Omega$-valued random variables on a probability space $(N,\mathcal{N},\mathbb{P})$. Suppose $X_i$ are iid and $Q$-distributed, $Y_i$ are iid and $P$-distributed, and let $Q_n,P_n$ denote the corresponding $n$-sample empirical measures. Suppose Assumption \ref{assump:Phi_k} holds for the  spaces $\Phi_k,\Psi\subset \mathcal{M}(\Omega)$, $k\in\mathbb{Z}_+$; in particular, the $\Phi_k$'s have the form
\begin{align}
\Phi_k=\{\phi_k(\cdot,\theta):\theta\in\Theta_k\}\,.
\end{align}
Define the corresponding estimator
\begin{align}\label{eq:estimator_k}
   \widehat{R}_\alpha^{n,k}(Q\|P) =\sup_{\theta\in\Theta_k}\left\{\frac{1}{\alpha-1}\log\left[\int  e^{(\alpha-1)\phi_{k,\theta}}dQ_n\right]-\frac{1}{\alpha}\log\left[\int e^{\alpha \phi_{k,\theta}}dP_n\right]\right\}\,.
\end{align}
 \begin{enumerate}
\item If $R_\alpha(Q\|P)<\infty$ then for all $\delta>0$ there exists $K\in\mathbb{Z}^+$ such that for all $k\geq K$ we have
\begin{align}
\lim_{n\to\infty}\mathbb{P}\left( \left|R_\alpha(Q\|P)-\widehat{R}_\alpha^{n,k}(Q\|P)\right|\geq\delta\right)=0\,.
\end{align}
\item If $R_\alpha(Q\|P)=\infty$ then for all $M>0$ there exists $K\in\mathbb{Z}^+$ such that for all $k\geq K$ we have
\begin{align}
\lim_{n\to\infty}\mathbb{P}\left( \widehat{R}_\alpha^{n,k}(Q\|P)\leq M\right)=0\,.
\end{align}

\end{enumerate}
\end{theorem}
The proof of Theorem \ref{thm:consistency}, which can be found in Section \ref{sec:consistency_proof}, is inspired by  the work in \cite{MINE_paper} which used the Donsker-Varadhan variational formula \eqref{eq:DV_var} to  estimate the KL-divergence. However, as mentioned above, we have developed new techniques that  allows us to prove consistency when $Q$ and $P$ to have non-compact support. This is accomplished by introducing  the space $\Psi$ in both Lemma \ref{lemma:Renyi_var_Phi} and Theorem \ref{thm:consistency}, which allows the use of $\phi$'s that are  $\Psi$-bounded, as opposed to simply being bounded.

If $\Theta_k\subset\mathbb{R}^{d_k}\cap\{\theta:\|\theta\|\leq K_k\}$ and $\phi_k$ is bounded by $M_k$ and is $L_k$\,-\,Lipschitz (i.e., Lipschitz continuous with constant $L_k$) in $\theta\in\Theta_k$   then one can derive sample complexity bounds for the estimator \eqref{eq:renyi_est_def}  by using the same technique that was used in \cite{MINE_paper} to study  KL-divergence estimators. To obtain an $\alpha$-divergence estimator  error less than $\epsilon$ with  probability at least $1-\delta$, it is sufficient to have the  number of samples, $n$, satisfy
\begin{align}\label{eq:n_lower_bound}
n\geq \frac{32 D_{\alpha,k}^2}{\epsilon^2}\left(d_k\log(16L_kK_k\sqrt{d_k}/\epsilon)+2d_kM_k\max\{|\alpha|,|\alpha-1|\}+\log(4/\delta)\right)\,,
\end{align}
where $D_{\alpha,k}\equiv\max\{e^{2|\alpha|M_k}/|\alpha|,e^{2|\alpha-1|M_k}/|\alpha-1|\}$. The qualitative behavior of \req{eq:n_lower_bound} in $\epsilon$, $\delta$, and $d_k$ is the same as the KL result from \cite{MINE_paper}, though some modifications to the proof are necessary.  The derivation uses the same techniques as the proof of Theorem 3 in \cite{MINE_paper}.  In particular, it relies on  a combination of concentration inequalities and covering theorems to obtain a non-asymptotic  uniform law of large numbers-type result; see  \cite{vershynin_2018} for  details on these tools. We include a proof of \eqref{eq:n_lower_bound} in Section \ref{sec:complexity}.

\section{Numerical Examples}\label{sec:numerical}

In this section we present several numerical examples of using the estimator \eqref{eq:estimator_k}; in practice, we search for the optimum in \eqref{eq:estimator_k}   via stochastic gradient descent (SGD) \cite{doi:10.1137/120880811,Ghadimi2016,10.5555/3304889.3305071}.  We take the  function space, $\Phi$, to be a neural network  family $\phi_\theta$, $\theta\in\Theta$,  with ReLU activation function, $\sigma(x)=\ReLU(x)\equiv\max\{x,0\}$. We used the AdamOptimizer method  \cite{kingma2014adam,2019arXiv190409237R}, an adaptive learning-rate SGD algorithm, to search for the optimum. All computations were performed in TensorFlow. 

\begin{figure}[ht]
  \centering
  
  \begin{minipage}[b]{0.45\linewidth}
\centering
  \includegraphics[width=\textwidth]{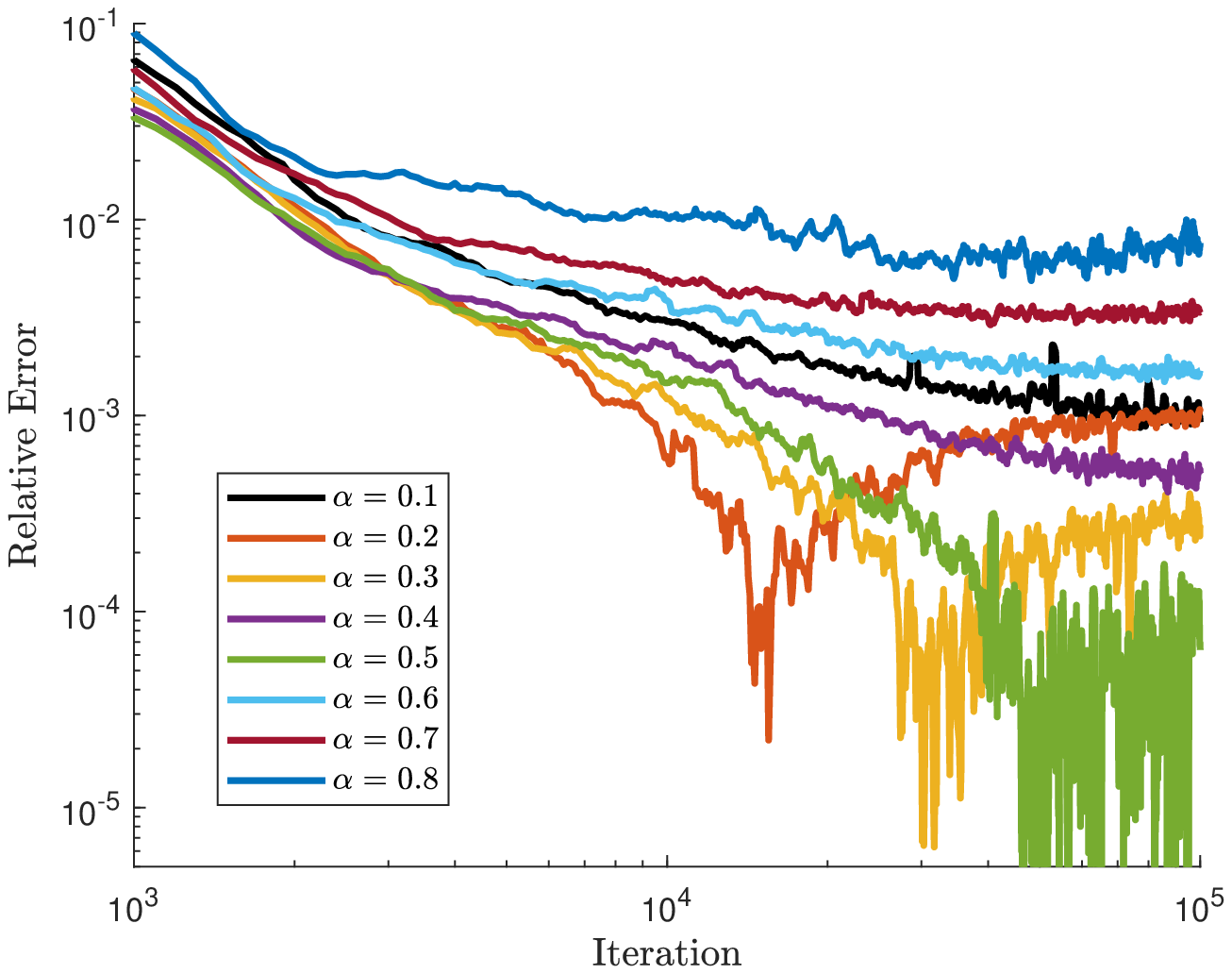}
\end{minipage}
\hspace{0.5cm}
\begin{minipage}[b]{0.45\linewidth}
\centering
  \includegraphics[width=\textwidth]{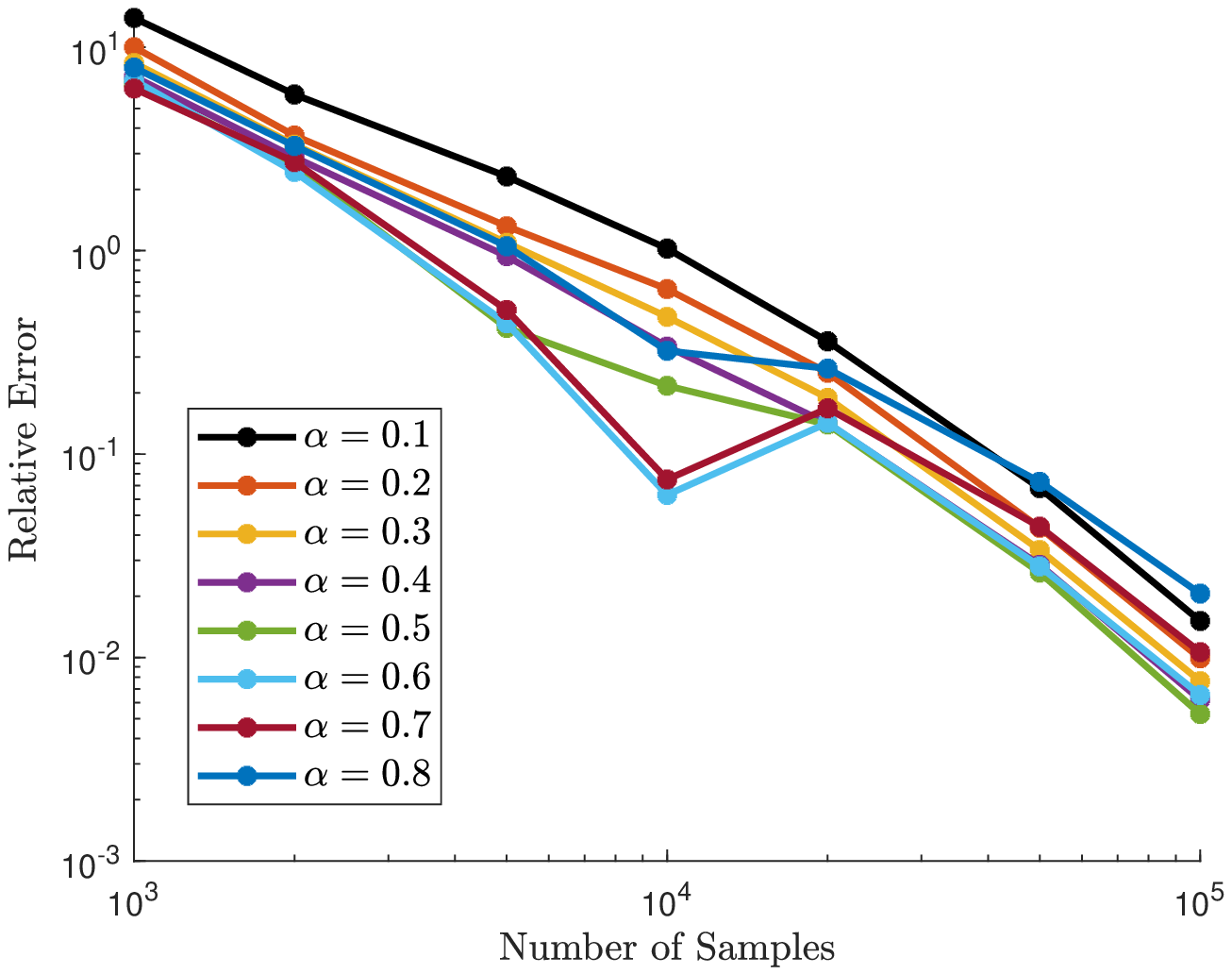}
\end{minipage}
   \caption{Left: Relative error of R{\'e}nyi divergence estimators \eqref{eq:estimator_k} between the distributions of $h(X)$ and $h(Y)$, where $X$ and $Y$ are $4$-dimensional Gaussians (with means $\mu_p=0$, $\mu_q=(2,0,0,0)$ and covariance matrices $\Sigma_p=I$, $\Sigma_q=diag(1.5,0.7,2,1)$) and $h:\mathbb{R}^4\to\mathbb{R}^{5000}$ is a nonlinear map.  Specifically, we let $h_i(x)=x_i$ for $i=1,...,4$ (to ensure it is an embedding) and then for $i>4$ we  define $h_i(x)=A_i(x)+c_{1,i}\cos(c_{2,i}x_{j_{1,i}})\sin(c_{3,i}x_{j_{2,i}})+c_{4,i}x_{j_{3,i}}x_{j_{4,i}}$, 
   where $A$ is an affine function  and $j_{k,i}\in\{1,...,4\}$; the parameters of $A$ and the $c_{k,i}$'s were randomly selected at the start of each run (all components are iid $N(0,1)$). The  indices $j_{k,i}$ were also randomly selected at the start of each run (iid $Unif(\{1,...,4\})$). Computations were done using   a neural network with 1 hidden layer of 128 nodes. On the left we show the relative error as a function of the number of SGD iterations; SGD was performed using a minibatch size of 1000 and an initial learning rate of $2\times 10^{-4}$. We show the moving average over the last 10 data points, with results  averaged over 20 runs. The behavior of the $\alpha=0.2, 0.3$ curves is due to the estimates crossing above and converging to a result slightly above the true values. On this problem the method failed to converge when $\alpha=0.9$ and when using the KL-divergence. Right: The relative error as a function of the number of samples, $N$.  We used a fixed number of 10000 SGD iterations, with the other parameters being as in the left panel.   Results were averaged over 100 runs.  The error is well approximated by a power-law decay of $N^{-1.4}$ and this behavior appears  insensitive to the value of $\alpha$.
   }\label{fig:Renyi_highD}
\end{figure}

 \subsection{Example: Estimating   R{\'e}nyi Divergences in High Dimensions}\label{sec:Renyi_HD_example}
  Estimators of divergences based on variational formulas are especially powerful in high dimensional systems with hidden low-dimensional (non-linear) structure, a setting that, again, is challenging for likelihood-ratio based  methods.  We illustrate the effectiveness of  the estimator \eqref{eq:estimator_k} in such a setting
 by estimating the R{\'e}nyi divergence between the distributions of $h(X)$ and $h(Y)$, where $X$ and $Y$ are both $4$-dimensional Gaussians and $h:\mathbb{R}^4\to \mathbb{R}^{5000}$ is a non-linear map.  If $h$ is an embedding (in particular, it must be one-to-one) then the data processing inequality (see Theorem 14 in \cite{1705001}) implies  $R_\alpha(P_{h(X)}\|P_{h(Y)})=R_\alpha(P_X\|P_Y)$, with the latter being easily computable (we use $P_Z$ to denote the distribution of a random variable $Z$).  Hence we have an exact value with which we can compare our numerical estimate of $R_\alpha(P_{h(X)}\|P_{h(Y)})$. In  Figure \ref{fig:Renyi_highD} we show the relative error,  comparing the results of our method to the exact values of the R{\'e}nyi divergences. The left panel shows the error  as a function of the number of SGD iterations and the right panel shows the error as a function of the size of the data set.     Our choice of nonlinear map $h$ is detailed in the caption.  We emphasize that the estimator \eqref{eq:renyi_est_def} is effective in high dimensions, with  no preprocessing (i.e., dimensional reduction) of the data   required; the results shown in Figure \ref{fig:Renyi_highD} were obtained by applying the algorithm  directly to the 5000-dimensional data.  Note that here, and as a general rule, the estimation becomes more difficult as $\alpha\to 0,1$ (i.e., the KL limits), regimes where the importance of rare events increases. The method failed to converge when $\alpha=1$ (i.e., when using the KL objective functional) and numerical estimation is even more challenging when $\alpha>1$.

\subsection{Example: Estimating R{\'e}nyi-Based Mutual Information}\label{sec:Renyi_MI}

Next we demonstrate the use of \eqref{eq:estimator_k} in the estimation of R{\'e}nyi mutual information,
\begin{flalign}\label{eq:Renyi_MI}
\text{\bf (R{\'e}nyi-MI)}&\hspace{4.8cm} R_\alpha(P_{(X,Y)}\|P_X\times P_Y)\,,&
\end{flalign} 
between random variables $X$ and $Y$; this should be compared with \cite{MINE_paper}, which used the Donsker-Varadhan variational formula to estimate KL mutual information, and \cite{birrell2020optimizing} which considered $f$-divergences.  (Mutual information is typically defined in terms of the KL-divergence, but one can consider many alternative divergences; see, e.g., \cite{f-sensitivity}).   In the left panel of Figure \ref{fig:Renyi_MI} we show the results of estimating the R{\'e}nyi-MI where $\alpha=1/2$ and $X$ and $Y$ are correlated  $20$-dimensional Gaussians   with component-wise correlation $\rho$ (the same case that was considered in \cite{MINE_paper,birrell2020optimizing}). This is a moderate dimensional problem (specifically, $40$-dimensional) with no low-dimensional structure. Our method is capable of accurately estimating the R{\'e}nyi-MI over a wide range of correlations,  something not achievable with likelihood-ratio based non-parametric methods (again, see \cite{NIPS2015_5911,MINE_paper}).

\begin{figure}[ht]

  \begin{minipage}[b]{0.45\linewidth}
\centering
  \includegraphics[width=\textwidth]{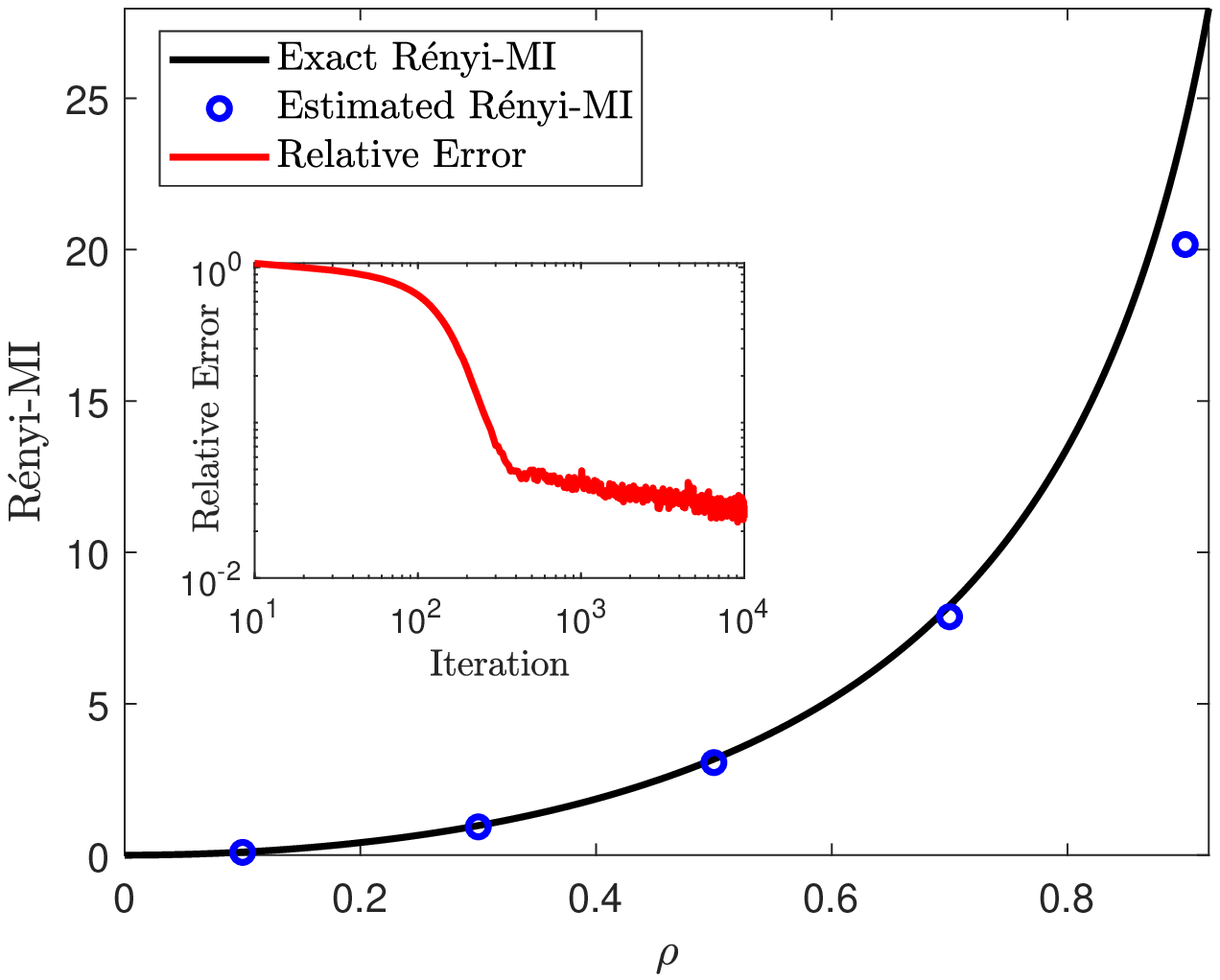}
\end{minipage}
\hspace{0.5cm}
\begin{minipage}[b]{0.45\linewidth}
\centering
  \includegraphics[width=\textwidth]{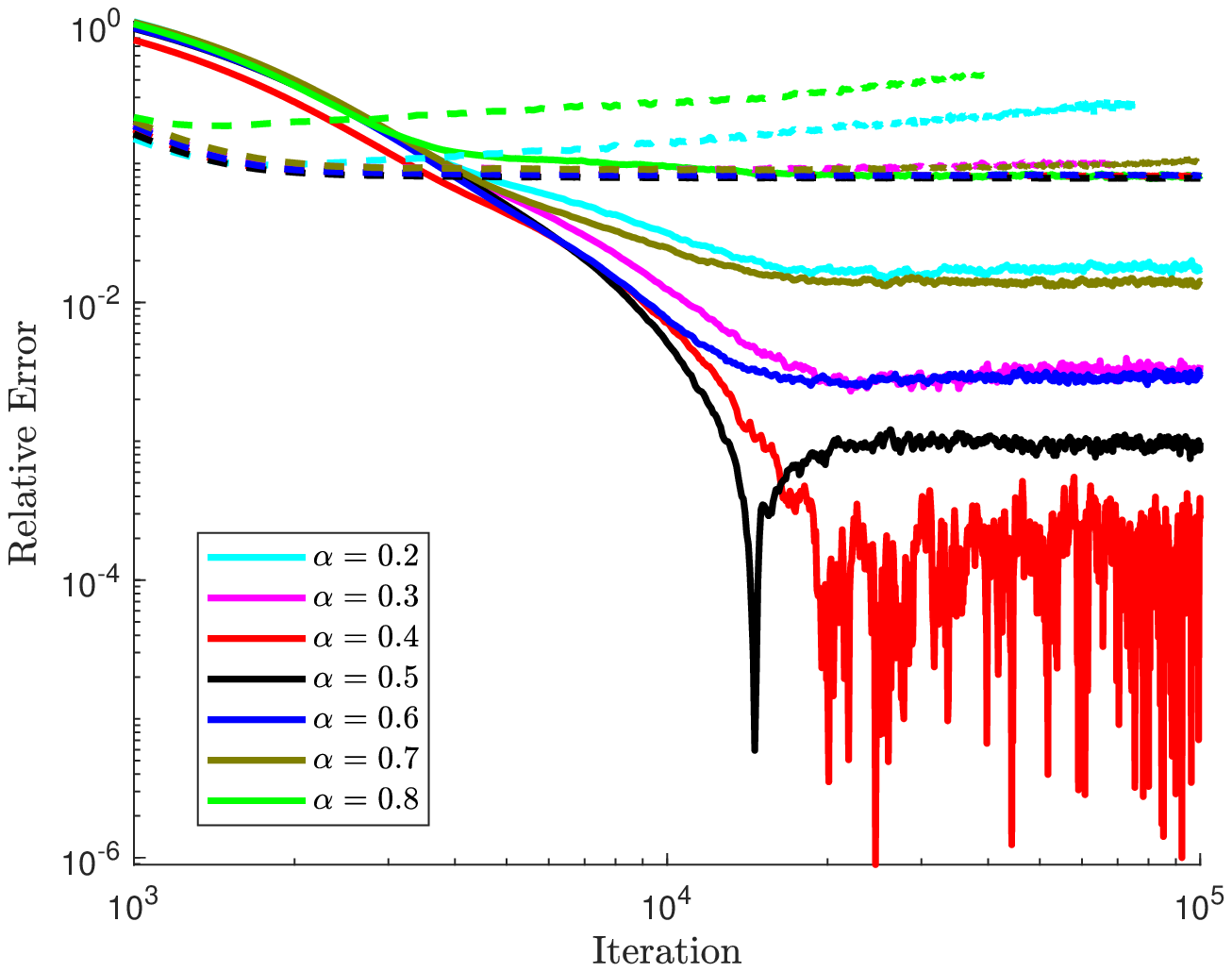}
\end{minipage}

   \caption{ Left: Estimation of R{\'e}nyi-based mutual information \eqref{eq:Renyi_MI} with $\alpha=1/2$  between  $20$-dimensional correlated Gaussians with component-wise correlation $\rho$. We used a  neural network with one hidden layer of 256 nodes and training was performed with a minibatch size of 1000. We show the R{\'e}nyi-MI as a function of $\rho$ after 10000 steps of SGD and averaged over 20 runs. The inset shows the relative error  for a single run with $\rho=0.5$, as a function of the number of SGD iterations.  Right: Estimation of the  R{\'e}nyi divergence between  two $25$-dimensional distributions of the form $\prod_{i=1}^{25}Beta(a_i,b_i)$.   The exponential family estimator \eqref{eq:exp_fam_est} (solid curves) outperformed the  neural-network estimator \eqref{eq:estimator_k} (dashed curves) with a comparable number of parameters (one  hidden layer with 4 nodes). Training was performed with a minibatch size of 1000 and an initial learning rate of $0.001$. Results were averaged over $20$ runs and the values of the $a$ and $b$ parameters for each distribution were randomly selected at the start of each run. Again, the estimation becomes more difficult as $\alpha\to 0,1$.}\label{fig:Renyi_MI}
\end{figure}

 \subsection{Example: Estimating R{\'e}nyi Divergence for Exponential Families}\label{sec:exp_family_ex}
 As discussed in Section \ref{sec:exp_family}, when working with an exponential family the formula for the optimizer (see Corollary \ref{cor:g_star}) reduces the R{\'e}nyi variational formula to a finite dimensional optimization problem (see \req{eq:Renyi_manifold}).  Using the corresponding estimator,
 \begin{align}\label{eq:exp_fam_est}
\widehat R^n_\alpha(Q\|P)=&\sup_{\Delta \kappa\in\mathbb{R}^{k} }\left\{\frac{1}{\alpha-1}\log\int e^{(\alpha-1)\Delta \kappa\cdot T(x)}dQ_n-\frac{1}{\alpha}\log\int e^{\alpha \Delta \kappa\cdot T(x)}dP_n\right\}\,,
 \end{align}
 can yield a substantial computational benefit over a  general-purpose neural-network estimator \eqref{eq:estimator_k}, as we now demonstrate. Here we estimate the divergence between products of Beta distributions; this is another moderate dimensional problem (specifically, $25$-dimensional) with no low dimensional structure.  The results are shown in the right panel of Figure \ref{fig:Renyi_MI}. The solid curves show the relative error that resulted from using \eqref{eq:exp_fam_est}, while the dashed curves show the result of using a neural-network estimator \eqref{eq:estimator_k} with a comparable number of parameters (specifically, one hidden layer with 4 nodes, and hence on the order of $100$ parameters).  The former achieves high accuracy over a range of $\alpha$'s while the latter performs poorly and fails to converge in several cases.  To achieve comparable accuracy with a neural-network estimator would require a much larger network, leading to a much greater computational cost.

\section{Proofs}

\subsection{Proof of the R{\'e}nyi-Donsker-Varadhan Variational Formula}\label{app:gen_DV}

The starting point for the proof of Theorem \ref{thm:gen_DV} is the following   variational formula, proven in \cite{atar2015robust}: Let  $P$ be a probability measure on $(\Omega,\mathcal{M})$, $g\in \mathcal{M}_b(\Omega)$, and  $\alpha>0$, $\alpha\neq 1$. Then
\begin{align}\label{eq:Renyi_MGF_sup}
&\frac{1}{\alpha}\log\!\left[\int e^{\alpha g} dP\right]=\sup_{Q }\left\{\frac{1}{\alpha-1}\log\!\left[\int e^{(\alpha-1) g}dQ \right]-R_{\alpha}(Q\|P)\right\},
\end{align}
where the optimization is over all probability measures, $Q$, on $(\Omega,\mathcal{M})$. (Let $\gamma=\alpha$, $\beta=\alpha-1$ in Eq. (1.3) of \cite{atar2015robust}). Though the right hand side of \req{eq:Renyi_MGF_sup} is not a Legendre transform, \eqref{eq:Renyi_MGF_sup} is still in some sense a `dual' version of \eqref{eq:Renyi_var}; this is reminiscent of the duality between the Donsker-Varadhan variational formula \eqref{eq:DV_var}  and the Gibbs variational principle (see Proposition 1.4.2 in \cite{Dupuis_Ellis}).  \req{eq:Renyi_MGF_sup} was previously used in \cite{atar2015robust,10.1214/19-AAP1468,atar2020robust} to derive uncertainty quantification bounds on risk-sensitive quantities (e.g., rare events or large deviations estimates) and in \cite{pmlr-v51-begin16} to derive PAC-Bayesian bounds.

In fact, we will not require the full strength of \eqref{eq:Renyi_MGF_sup}.  We will only need the following bound for $g\in \mathcal{M}_b(\Omega)$, $\alpha>0$, $\alpha\neq 1$:
\begin{align}\label{eq:Renyi_MGF_sup_ub}
&\frac{1}{\alpha-1}\log\!\left[\int e^{(\alpha-1) g}dQ \right]\leq\frac{1}{\alpha}\log\!\left[\int e^{\alpha g} dP\right]+R_{\alpha}(Q\|P)\,.\
\end{align}
 To keep our argument self-contained, we include a proof of \req{eq:Renyi_MGF_sup_ub} below. Our proof is adapted from the proof of \eqref{eq:Renyi_MGF_sup} found  in Section 4 of \cite{atar2015robust}. We note that an alternative proof of \req{eq:Renyi_MGF_sup_ub} can be given by using a different variational formula for the R{\'e}nyi divergences, which can be found in Theorem 30 of \cite{van2014renyi} and also in Theorem 1 of \cite{8006657}.
\begin{proof}[Proof of \req{eq:Renyi_MGF_sup_ub}]
We separate the proof into two cases.\\
1) $\alpha>1$: If $Q\not\ll P$ the result is trivial (see \req{eq:Renyi_formula}), so   assume $Q\ll P$.  For $g\in\mathcal{M}_b(\Omega)$ we can use H{\"o}lder's inequality with conjugate exponents $\alpha/(\alpha-1)$ and $\alpha$ to obtain
\begin{align}   
    \frac{1}{\alpha-1}\log\int e^{(\alpha-1)g}dQ  \leq&\frac{1}{\alpha-1}\log\left[\left(\int (e^{(\alpha-1)g})^{\frac{\alpha}{\alpha-1}}dP\right)^{\frac{\alpha-1}{\alpha}}\!\!\left(\int \left(\frac{dQ}{dP}\right)^\alpha \!\!dP\right)^{\frac{1}{\alpha}}\right]\\
    =&\frac{1}{\alpha}\log \int e^{\alpha g}dP+\frac{1}{\alpha(\alpha-1)}\log \int (dQ/dP)^\alpha dP\,.\notag
\end{align}
In this case the definition \eqref{eq:Renyi_formula} implies $R_\alpha(Q\|P)=\frac{1}{\alpha(\alpha-1)}\log \int (dQ/dP)^\alpha dP$ and so we have proven the claimed bound \eqref{eq:Renyi_MGF_sup_ub}.

2) $\alpha\in(0,1)$: Let $dP=pd\nu$, $dQ=qd\nu$ as in  definition \eqref{eq:Renyi_formula} and define $h= e^{-g} q$. Then
\begin{align}\label{eq:renyi_h}
R_\alpha(Q\|P)=&\frac{1}{\alpha(\alpha-1)}\log \int q^\alpha p^{1-\alpha}d\nu=\frac{1}{\alpha(\alpha-1)}\log \int_{p,q>0} (h/p)^{\alpha-1}  e^{(\alpha -1)g}dQ\,.
\end{align}
Using H{\"o}lder's inequality for the measure $e^{(\alpha-1)g}dQ$, the conjugate exponents $1/\alpha$ and $1/(1-\alpha)$, and the functions $1$ and $1_{q,p>0}(h/p)^{\alpha-1}$ we find
\begin{align}
 \int_{q,p>0}(h/p)^{\alpha-1}  e^{(\alpha -1)g}dQ \leq& \left(\int
   e^{(\alpha -1)g}dQ\right)^\alpha  \left(\int_{q,p>0} (h/p)^{-1} e^{(\alpha -1)g}dQ\right)^{1-\alpha}\\
      =&\left(\int
   e^{(\alpha -1)g}dQ\right)^\alpha  \left(\int_{q,p>0}  e^{\alpha g }dP\right)^{1-\alpha}\notag\\
   \leq&\left(\int
   e^{(\alpha -1)g}dQ\right)^\alpha  \left(\int  e^{\alpha g }dP\right)^{1-\alpha}\,.\notag
\end{align}
Taking the logarithm of both sides, dividing by $\alpha(\alpha-1)$ (which is negative), and using \req{eq:renyi_h} we arrive at
\begin{align}
    R_\alpha(Q\|P)\geq \frac{1}{\alpha-1}\log\int e^{(\alpha-1)g}dQ-\frac{1}{\alpha}\log\int e^{\alpha g}dP\,.
\end{align}
This implies the claimed bound \eqref{eq:Renyi_MGF_sup_ub} and completes the proof.
\end{proof}
We now use \req{eq:Renyi_MGF_sup_ub} to derive the variational formula \eqref{eq:Renyi_var}. The argument is inspired by the proof of the  Donsker-Varadhan variational formula from Appendix C.2 in \cite{Dupuis_Ellis}.
\begin{proof}[Proof of Theorem \ref{thm:gen_DV}]
First let $\Gamma=\mathcal{M}_b(\Omega)$. If one can show \req{eq:Renyi_var} for all $\alpha>1$ and all $P,Q$, then, using  \req{eq:R_neg_alpha} and reindexing $g\to -g$ in the supremum, one finds that \req{eq:Renyi_var} also holds for all $\alpha<0$.  So we only need to consider the cases $\alpha\in(0,1)$ and $\alpha>1$.

 \req{eq:Renyi_MGF_sup_ub} immediately implies
\begin{align}
R_\alpha(Q\|P)
\geq&\sup_{g\in \mathcal{M}_b(\Omega)}\left\{\frac{1}{\alpha-1}\log\!\left[\int e^{(\alpha-1)g}dQ\right]-\frac{1}{\alpha}\log\!\left[\int e^{\alpha g}dP\right]\right\}\notag\\
\equiv& \widetilde{R}_\alpha(Q\|P)\,.\label{eq:tildeR_def}
\end{align}
If $Q\ll P$ and $g^*\equiv\log(dQ/dP)\in\mathcal{M}_b(\Omega)$ then the reverse inequality easily follows from an explicit calculation. However, $g^*\in\mathcal{M}_b(\Omega)$ is a very strong assumption which we do not make here.  Our general proof will therefore require several limiting arguments, but will still be based on this intuition.

We separate the proof of the reverse inequality into three cases.\\
1) $\alpha>1$ and $Q\not\ll P$: We will show $\widetilde{R}_\alpha(Q\|P)=\infty$, which will prove the desired inequality.  To do this, take a measurable set $A$ with $P(A)=0$ but $Q(A)\neq 0$ and define $g_n=n1_A$.  The definition (\ref{eq:tildeR_def}) implies
\begin{align}
\widetilde{R}_\alpha(Q\|P)\geq& \frac{1}{\alpha-1}\log\int e^{(\alpha-1)g_n}dQ-\frac{1}{\alpha}\log\int e^{\alpha g_n}dP\\
=&\frac{1}{\alpha-1}\log\!\left[ e^{(\alpha-1)n}Q(A)+Q(A^c)\right]-\frac{1}{\alpha}\log P(A^c)\,.\notag
\end{align}
The lower bound goes to $+\infty$ as $n\to\infty$ (here it is key that $\alpha>1$) and therefore we have the claimed result.

2) $\alpha>1$ and $Q\ll P$: In this case we can take $\nu=P$ in \req{eq:Renyi_formula} and write
\begin{align}
R_\alpha(Q\|P)=\frac{1}{\alpha(\alpha-1)}\log\!\left[\int (dQ/dP)^\alpha dP\right].
\end{align}
Define
\begin{align}\label{eq:f_nm_def}
f_{n,m}(x)=x1_{1/m<x<n}+n1_{x\geq n}+1/m1_{x\leq 1/m}
\end{align}
 and $g_{n,m}=\log(f_{n,m}(dQ/dP))$.  These are bounded and so  \req{eq:tildeR_def} implies
\begin{align}\label{eq:R_lb_alpha_g1}
\widetilde{R}_\alpha(Q\|P)
\geq& \frac{1}{\alpha-1}\log\int e^{(\alpha-1)g_{n,m}}dQ-\frac{1}{\alpha}\log\int e^{\alpha g_{n,m}}dP\\
=&\frac{1}{\alpha-1}\log\int f_{n,m}(dQ/dP)^{(\alpha-1)}\frac{dQ}{dP}dP-\frac{1}{\alpha}\log\int f_{n,m}(dQ/dP)^\alpha dP\,.\notag
\end{align}
Define  $f_{n,\infty}(x)=x1_{x<n}+n1_{x\geq n}$. Using the dominated convergence theorem to take $m\to\infty$ in \eqref{eq:R_lb_alpha_g1} we find 
\begin{align}
\widetilde{R}_\alpha(Q\|P)\geq&\frac{1}{\alpha-1}\log\int f_{n,\infty}(dQ/dP)^{(\alpha-1)}\frac{dQ}{dP}dP-\frac{1}{\alpha}\log\int f_{n,\infty}(dQ/dP)^\alpha dP\\
\geq&\frac{1}{\alpha(\alpha-1)}\log\int f_{n,\infty}(dQ/dP)^\alpha dP\,.\notag
\end{align}
To obtain the last line we used  $xf_{n,\infty}(x)^{\alpha-1}\geq f_{n,\infty}(x)^\alpha$. Next, we have $0\leq f_{n,\infty}(dQ/dP)\nearrow dQ/dP$ as $n\to\infty$, and so the monotone convergence theorem  implies
\begin{align}
\widetilde{R}_\alpha(Q\|P)\geq&\frac{1}{\alpha(\alpha-1)}\log\int(dQ/dP)^\alpha dP=R_\alpha(Q\|P)\,.
\end{align}
This proves the claimed result  for case 2.

3) $\alpha\in(0,1)$: In this case definition \eqref{eq:Renyi_formula} becomes
\begin{align}
R_\alpha(Q\|P)=\frac{1}{\alpha(\alpha-1)}\log\!\left[\int_{p>0} q^\alpha p^{1-\alpha}d\nu\right]\,,
\end{align}
where $\nu$ is any sigma-finite positive measure for which $dQ=qd\nu$ and $dP=pd\nu$. Define $f_{n,m}(x)$ via \req{eq:f_nm_def} and let $g_{n,m}=\log(f_{n,m}(q/p))$, where $q/p$ is defined to be $0$ if $q=0$ and $+\infty$ if $p=0$ and $q\neq 0$. The functions $g_{n,m}$ are bounded, hence \req{eq:tildeR_def} implies
\begin{align}\label{eq:R_lb_alpha_l1}
\widetilde{R}_\alpha(Q\|P)\geq& -\frac{1}{1-\alpha}\log\int e^{(\alpha-1)g_{n,m}}dQ-\frac{1}{\alpha}\log\int e^{\alpha g_{n,m}}dP\\
=&-\frac{1}{1-\alpha}\log\int f_{n,m}(q/p)^{\alpha-1}qd\nu-\frac{1}{\alpha}\log\int f_{n,m}(q/p)^\alpha pd\nu\,.\notag
\end{align} 
Define $f_{\infty,m}(x)=x1_{x> 1/m}+1/m1_{x\leq1/m}$. We have the bound  $f_{n,m}(q/p)^{\alpha-1}\leq (1/m)^{\alpha-1}$ (here it is critical that $\alpha\in(0,1)$) and so the dominated convergence theorem can be used to compute the $n\to\infty$ limit of the first term on the right hand side of \eqref{eq:R_lb_alpha_l1}, while the second term can be bounded using $f_{n,m}(q/p)^\alpha \leq f_{\infty,m}(q/p)^\alpha$.  We thereby obtain  
\begin{align}\label{eq:H_lb}
\widetilde{R}_\alpha(Q\|P)\geq&-\frac{1}{1-\alpha}\log\int f_{\infty,m}(q/p)^{\alpha-1}qd\nu-\frac{1}{\alpha}\log\int f_{\infty,m}(q/p)^\alpha pd\nu\\
\geq&-\frac{1}{1-\alpha}\log\int_{q>0,p>0} q^\alpha p^{1-\alpha}d\nu-\frac{1}{\alpha}\log\int_{p>0} f_{\infty,m}(q/p)^\alpha pd\nu\,,\notag
\end{align}
where we used $f_{\infty,m}(x)\geq x$ to obtain the second line.  Using the dominated convergence theorem   on the second term (which is always finite) we find
\begin{align}
\widetilde{R}_\alpha(Q\|P)\geq&-\frac{1}{1-\alpha}\log\int_{p>0} q^\alpha p^{1-\alpha}d\nu-\frac{1}{\alpha}\log\int_{p>0} q^\alpha p^{1-\alpha}d\nu\notag\\
=&\frac{1}{\alpha(\alpha-1)}\log\int_{p>0} q^\alpha p^{1-\alpha}d\nu=R_\alpha(Q\|P)\,. \notag
\end{align}
Therefore the claim is proven in case 3, and  the proof  of \req{eq:Renyi_var} is complete.

In addition, now suppose that $(\Omega,\mathcal{M})$ is a metric space with the Borel $\sigma$-algebra.  We will next show that \eqref{eq:Renyi_var} holds with $\Gamma=C_b(\Omega)$, the space of bounded continuous functions on $\Omega$. Define the probability measure $\mu=(P+Q)/2$ and let $g\in\mathcal{M}_b(\Omega)$. Lusin's theorem (see, e.g., Appendix D in \cite{dudley2014uniform}) implies that for all $n\in\mathbb{Z}^+$ there exists a closed set $F_n\subset \Omega$ such that  $\mu(F_n^c)<1/n$ and $g|_{F_n}$ is continuous. By the Tietze Extension Theorem (see, e.g., Theorem 4.16 in \cite{folland2013real}) there exists $g_n\in C_b(\Omega)$ with $\|g_n\|_\infty\leq \|g\|_\infty$ and $g_n=g$ on $F_n$. Therefore 
\begin{align}\label{eq:Cb_reduction1}
\left|\int e^{(\alpha-1)g_n}dQ-\int e^{(\alpha-1) g}dQ\right|
\leq& (\|e^{(\alpha-1)g_n}\|_\infty+\|e^{(\alpha-1)g}\|_\infty)Q(F_n^c)\\
\leq& 4e^{|\alpha-1|\|g\|_\infty}/n\to 0\notag
\end{align}
as $n\to\infty$.  Similarly, we have $\lim_{n\to\infty}\int e^{\alpha g_n}dP= \int e^{\alpha g}dP$. Hence
\begin{align}\label{eq:Cb_reduction2}
 &\sup_{g\in C_b(\Omega)}\!\left\{\frac{1}{\alpha-1}\log\!\left[\int e^{(\alpha-1)g}dQ\right]-\frac{1}{\alpha}\log\!\left[\int e^{\alpha g}dP\right]\right\}\\
 \geq&\lim_{n\to\infty}\left(\frac{1}{\alpha-1}\log\!\left[\int e^{(\alpha-1)g_n}dQ\right]-\frac{1}{\alpha}\log\!\left[\int e^{\alpha g_n}dP\right]\right)\notag\\
 =&\frac{1}{\alpha-1}\log\!\left[\int e^{(\alpha-1)g}dQ\right]-\frac{1}{\alpha}\log\!\left[\int e^{\alpha g}dP\right]\,.\notag
\end{align}
$g\in\mathcal{M}_b(\Omega)$ was arbitrary and so we have proven
\begin{align}\label{eq:Cb_reduction3}
&\sup_{g\in C_b(\Omega)}\!\left\{\frac{1}{\alpha-1}\log\!\left[\int e^{(\alpha-1)g}dQ\right]-\frac{1}{\alpha}\log\!\left[\int e^{\alpha g}dP\right]\right\}\\
\geq&\!\! \sup_{g\in\mathcal{M}_b(\Omega)}\!\left\{\frac{1}{\alpha-1}\log\!\left[\int e^{(\alpha-1)g}dQ\right]-\frac{1}{\alpha}\log\!\left[\int e^{\alpha g}dP\right]\right\}.\notag
\end{align}
The reverse inequality is trivial. Therefore we have shown that  \eqref{eq:Renyi_var} holds with $\Gamma=C_b(\Omega)$. To see that  \eqref{eq:Renyi_var} holds when $\Gamma=  \Lip_b(\Omega)$, use   the fact that every $g\in C_b(\Omega)$ is the pointwise limit of Lipschitz functions, $g_n$, with $\|g_n\|_\infty\leq \|g\|_\infty$ (see Box 1.5 on page 6 of \cite{santambrogio2015optimal}). The result then follows from a similar computation to the above, this time using the dominated convergence theorem.

Finally, we prove \eqref{eq:Renyi_var} with $\Gamma=\mathcal{M}(\Omega)$. To do this we need to show  
\begin{align}\label{eq:Renyi_lb}
R_\alpha(Q\|P)\geq&\frac{1}{\alpha-1}\log\!\left[\int e^{(\alpha-1)g}dQ\right]-\frac{1}{\alpha}\log\!\left[\int e^{\alpha g}dP\right]
\end{align}
for all  $g\in\mathcal{M}(\Omega)$.  The equality \eqref{eq:Renyi_var}  then  follows by combining \req{eq:Renyi_lb} with  Theorem \ref{thm:gen_DV}.  To prove the bound \eqref{eq:Renyi_lb} we start by fixing  $g\in\mathcal{M}(\Omega)$ and defining the truncated functions $g_{n,m}=-n 1_{g<-n}+g1_{-n\leq g\leq m}+m1_{g>m}$. These are bounded and so Theorem \ref{thm:gen_DV} implies
\begin{align}\label{eq:R_gnm_bound}
R_\alpha(Q\|P)\geq& \frac{1}{\alpha-1}\log\!\left[\int e^{(\alpha-1)g_{n,m}}dQ\right]-\frac{1}{\alpha}\log\!\left[\int e^{\alpha g_{n,m}}dP\right].
\end{align}
We now consider three cases, based on the value of $\alpha$.

1)  $\alpha>1$:   If  $\int e^{\alpha g}dP=\infty$ then \req{eq:Renyi_lb} is trivial (due to our convention that $\infty-\infty=-\infty$, this is true even if $\int e^{(\alpha-1)g}dQ=\infty$), so suppose $\int e^{\alpha g}dP<\infty$.  When $\alpha>1$,  \req{eq:R_gnm_bound} involves integrals of the form $\int e^{c g_{n,m}}d\mu$ where $c> 0$ and $\mu$ is a probability measure. We have $\lim_{n\to\infty}e^{cg_{n,m}}=e^{cg_m}$ where $g_m\equiv g1_{ g\leq m}+m1_{g>m}$ and $e^{cg_{n,m}}\leq e^{cm}$ for all $n$. Therefore the dominated convergence theorem implies
\begin{align}
\lim_{n\to\infty}\int e^{c g_{n,m}}d\mu=\int e^{c g_{m}}d\mu\,.
\end{align}
We have $0\leq e^{cg_m}\nearrow e^{cg}$ as $m\to \infty$ and hence the monotone convergence theorem yields
\begin{align}
\lim_{m\to\infty}\lim_{n\to\infty}\int e^{c g_{n,m}}d\mu=\lim_{m\to\infty}\int e^{c g_{m}}d\mu=\int e^{cg}d\mu\,.
\end{align}
Therefore we can take the iterated limit of \req{eq:R_gnm_bound} to obtain
\begin{align}
R_\alpha(Q\|P)
\geq& \frac{1}{\alpha-1}\log\!\left[\int e^{(\alpha-1)g}dQ\right]-\frac{1}{\alpha}\log\!\left[\int e^{\alpha g}dP\right]
\end{align}
(note that we are in the sub-case where the second term is finite, and so this is true even if $\int e^{(\alpha-1)g}dQ=\infty$). This proves the claim in case 1.

2) $\alpha<0$: Use \req{eq:R_neg_alpha} and  apply the result of  case 1 to the function $-g$ to obtain \eqref{eq:Renyi_lb}.

3) $0<\alpha<1$: If either $\int e^{(\alpha-1)g}dQ=\infty$ or $\int e^{\alpha g}dP=\infty$ then the bound (\ref{eq:Renyi_lb}) is again trivial, so suppose they are both finite.  For $c\in\mathbb{R}$ we can bound $e^{cg_{n,n}}\leq 1+e^{c g}$ and $\lim_{n\to\infty}e^{cg_{n,n}}= e^{cg}$.  Therefore the dominated convergence theorem implies that
\begin{align}
R_\alpha(Q\|P)
\geq&\lim_{n\to\infty}\left( \frac{1}{\alpha-1}\log\!\left[\int e^{(\alpha-1)g_{n,n}}dQ\right]-\frac{1}{\alpha}\log\!\left[\int e^{\alpha g_{n,n}}dP\right]\right)\\
=& \frac{1}{\alpha-1}\log\!\left[\int e^{(\alpha-1)g}dQ\right]-\frac{1}{\alpha}\log\!\left[\int e^{\alpha g}dP\right].\notag
\end{align}
This proves \req{eq:Renyi_lb} in case 3 and thus completes the proof of \req{eq:Renyi_var} when $\Gamma=\mathcal{M}(\Omega)$. \req{eq:Renyi_var} for the spaces between $\mathcal{M}_b(\Omega)$ (or $\Lip_b(\Omega)$) and $\mathcal{M}(\Omega)$ then easily follows.
\end{proof}

We end this subsection by deriving a formula for the optimizer.
\begin{proof}[Proof of Corollary \ref{cor:g_star}]

 If $Q\ll P$, $dQ/dP>0$, and $(dQ/dP)^\alpha\in L^1(P)$ then we also have $P\ll Q$.  By taking $\nu=P$ in \eqref{eq:Renyi_formula} (and for $\alpha<0$, using the definition \eqref{eq:R_neg_alpha}) we find
\begin{align}
R_\alpha(Q\|P)=\frac{1}{\alpha(\alpha-1)}\log \int\left({dQ}/{dP}\right)^\alpha dP\,.
\end{align}
Letting $g^*=\log dQ/dP$, it is straightforward to show by direct calculation that
\begin{align}
&\frac{1}{\alpha-1}\log\!\left[\int e^{(\alpha-1)g^*}dQ\right]-\frac{1}{\alpha}\log\!\left[\int e^{\alpha g^*}dP\right]=\frac{1}{\alpha(\alpha-1)}\log \int\left({dQ}/{dP}\right)^\alpha dP\,.
\end{align}
This, together with Theorem \ref{thm:gen_DV}, implies that \req{eq:Renyi_var} holds for any $\Gamma$ with $g^*\in\Gamma\subset\mathcal{M}(\Omega)$ and $g^*$ is an optimizer. This completes the proof.
\end{proof}
\subsection{Consistency Proof}\label{sec:consistency_proof}
In this subsection we prove consistency of the R{\'e}nyi divergence estimator \eqref{eq:estimator_k}.
\begin{proof}[Proof of Lemma \ref{lemma:Renyi_var_Phi}]
Both assumptions 1a and 2a imply that for $\phi\in\Phi$ there exists $\psi\in\Psi$ with $|\phi|\leq \psi$.  Either of the integrability assumptions 1b - 1c or 2b - 2c then imply that all expectations on the right hand side of \req{eq:var_formula_Phi} are finite.  Define the probability measure $\mu=(P+Q)/2$.  $\Omega$ is  a complete separable metric space, hence $\mu$ is inner regular.  In particular, for any $\delta>0$ there exists a compact set $K_\delta$ such that $\mu(K_\delta)>1-\delta$.  Fix  $g\in\Lip_b(\Omega)$. Assumptions 1a and 2a imply that  there exists $\psi_g\in\Psi$ such that $|g|\leq \psi_g$ and for all $\delta,\epsilon>0$ there exists $\phi_{\delta,\epsilon}\in\Phi$ with $|\phi_{\delta,\epsilon}|\leq \psi_g$ and, in the case of 1a, 
\begin{align}\label{eq:1a_approx}
    \sup_{x\in K_\delta}|g(x)-\phi_{\delta,\epsilon}(x)|<\epsilon\,,
    \end{align}
    while in the case of 2a we have
\begin{align}\label{eq:2a_approx}
    \max\left\{\left(\int_{K_\delta}|g-\phi_{\delta,\epsilon}|^pdQ\right)^{1/p},\left(\int_{K_\delta}|g-\phi_{\delta,\epsilon}|^pdP\right)^{1/p}\right\}<\epsilon\,.
\end{align}

The fact that $g$ and $\phi_{\delta,\epsilon}$ are bounded by $\psi_g$ implies
\begin{align}
&\int e^{(\alpha-1)\phi_{\delta,\epsilon}}dQ,\int e^{(\alpha-1)g}dQ\in[M_{g,-},M_{g,+}]\,,\,\,\,\,\,\int e^{\alpha\phi_{\delta,\epsilon}}dP, \int e^{\alpha g}dP\in [N_{g,-},N_{g,+}]\,,
\end{align}
where $M_{g,\pm}\equiv \int e^{\pm|\alpha-1|\psi_g }dQ\in(0,\infty)$, $N_{g,\pm}\equiv\int e^{\pm|\alpha|\psi_g} dP\in(0,\infty)$.  Using the fact that  $\log$ is $1/c$-Lipschitz on $[c,\infty)$ for all $c>0$  we can compute
\begin{align}
&\left|\frac{1}{\alpha-1}\log\int e^{(\alpha-1)g}dQ-\frac{1}{\alpha}\log\int e^{\alpha g}dP\right.\\
&\left.-\left(\frac{1}{\alpha-1}\log\int e^{(\alpha-1)\phi_{\delta,\epsilon}}dQ-\frac{1}{\alpha}\log\int e^{\alpha \phi_{\delta,\epsilon}}dP\right)\right|\notag\\
\leq &\frac{1}{|\alpha-1|M_{g,-}}\left|\int e^{(\alpha-1)g}dQ-\int e^{(\alpha-1)\phi_{\delta,\epsilon}}dQ\right|+\frac{1}{|\alpha|N_{g,-}}\left|\int e^{\alpha g}dP-\int e^{\alpha \phi_{\delta,\epsilon}}dP\right|\notag\\
\leq &\frac{1}{|\alpha-1|M_{g,-}}\int_{K_\delta} \left|e^{(\alpha-1)g}- e^{(\alpha-1)\phi_{\delta,\epsilon}}\right|dQ+\frac{2}{|\alpha-1|M_{g,-}}\int e^{|\alpha-1|\psi_g}1_{K_\delta^c}dQ\notag\\
&+\frac{1}{|\alpha|N_{g,-}}\int_{K_\delta} \left|e^{\alpha g}- e^{\alpha \phi_{\delta,\epsilon}}\right|dP+\frac{2}{|\alpha|N_{g,-}}\int e^{|\alpha|\psi_g}1_{K_\delta^c}dP\,.\notag
\end{align}
 Under assumption 1a, and restricting to $\epsilon\leq 1$ we can use \eqref{eq:1a_approx} to bound $|\phi_{\delta,\epsilon}|\leq \|g\|_\infty+1$ on $K_\delta$ and so $|e^{cg}-e^{c\phi_{\delta,\epsilon}}|1_{K_\delta}\leq |c|e^{|c|(\|g\|_\infty+1)}\epsilon$ for $c\in\mathbb{R}$.   Under assumption 2a we can use \eqref{eq:2a_approx} and H{\"o}lder's inequality to bound
 \begin{align}
 &\frac{1}{|\alpha-1|M_{g,-}}\int_{K_\delta} \left|e^{(\alpha-1)g}- e^{(\alpha-1)\phi_{\delta,\epsilon}}\right|dQ+\frac{1}{|\alpha|N_{g,-}}\int_{K_\delta} \left|e^{\alpha g}- e^{\alpha \phi_{\delta,\epsilon}}\right|dP\\
\leq &\frac{1}{ M_{g,-}}\int_{K_\delta} e^{|\alpha-1|\psi_g}|g-\phi_{\delta,\epsilon}|dQ+\frac{1}{N_{g,-}}\int_{K_\delta}e^{|\alpha|\psi_g}|g-\phi_{\delta,\epsilon}|dP\notag\\
\leq&\frac{1}{ M_{g,-}}\left(\int  e^{q|\alpha-1|\psi_g}dQ\right)^{1/q}\epsilon+\frac{1}{N_{g,-}}\left(\int e^{q|\alpha|\psi_g}dP\right)^{1/q}\epsilon\,.\notag
\end{align}
 In either case, we find
\begin{align}
&\frac{1}{\alpha-1}\log\int e^{(\alpha-1)g}dQ-\frac{1}{\alpha}\log\int e^{\alpha g}dP\\
\leq &\sup_{\phi\in\Phi}\left\{\frac{1}{\alpha-1}\log\int e^{(\alpha-1)\phi}dQ-\frac{1}{\alpha}\log\int e^{\alpha \phi}dP\right\}+D_{\delta,\epsilon}\,,\notag\\
&D_{\delta,\epsilon}\equiv D_g\epsilon+\frac{2}{|\alpha-1|M_{g,-}}\int_{K_\delta^c} e^{|\alpha-1|\psi_g}dQ+\frac{2}{|\alpha|N_{g,-}}\int_{K_\delta^c} e^{|\alpha|\psi_g}dP\,,\notag
\end{align}
where $D_g\in(0,\infty)$ is given by
\begin{align}
    D_g=M_{g,-}^{-1} e^{|\alpha-1|(\|g\|_\infty+1)}+N_{g,-}^{-1} e^{|\alpha|(\|g\|_\infty+1)}
\end{align}
 under assumption 1 and by
 \begin{align}
    D_g=M_{g,-}^{-1}\left(\int  e^{q|\alpha-1|\psi_g}dQ\right)^{1/q}+N_{g,-}^{-1}\left(\int e^{q|\alpha|\psi_g}dP\right)^{1/q}
\end{align}
 under assumption 2. Under either set of assumptions we have $e^{|\alpha-1|\psi_g}\in L^1(Q)$ and $e^{|\alpha|\psi_g}\in L^1(P)$. Combining this fact with $Q(K_\delta^c),P(K^c_\delta)\leq 2\delta$ we can use the dominated convergence theorem for convergence in measure to compute
\begin{align}
\lim_{\delta\searrow 0}\int_{K_\delta^c}e^{|\alpha-1|\psi_g}dQ=0=\lim_{\delta\searrow 0}\int_{K_\delta^c}e^{|\alpha|\psi_g}dP
\end{align}
(here   it is important that $\psi_g$ is independent of $\delta$). Therefore taking $\epsilon,\delta\searrow 0$ we  obtain
\begin{align}
&\frac{1}{\alpha-1}\log\int e^{(\alpha-1)g}dQ-\frac{1}{\alpha}\log\int e^{\alpha g}dP\\
\leq &\sup_{\phi\in\Phi}\left\{\frac{1}{\alpha-1}\log\int e^{(\alpha-1)\phi}dQ-\frac{1}{\alpha}\log\int e^{\alpha \phi}dP\right\}\,.\notag
\end{align}
This holds for all $g\in \Lip_b(\Omega)$ and so
\begin{align}\label{eq:Phi_bound}
&\sup_{g\in\Lip_b(\Omega)}\left\{\frac{1}{\alpha-1}\log\int e^{(\alpha-1)g}dQ-\frac{1}{\alpha}\log\int e^{\alpha g}dP\right\}\\
\leq &\sup_{\phi\in\Phi}\left\{\frac{1}{\alpha-1}\log\int e^{(\alpha-1)\phi}dQ-\frac{1}{\alpha}\log\int e^{\alpha \phi}dP\right\}\,.\notag
\end{align}
Using Theorem \ref{eq:Renyi_var} with $\Gamma=\Lip_b(\Omega)$ we see that the left hand side of \eqref{eq:Phi_bound} equals $R_\alpha(Q\|P)$. Theorem   \ref{thm:gen_DV} with $\Gamma=\mathcal{M}(\Omega)$ implies that the right hand side of \eqref{eq:Phi_bound} is bounded above by $R_\alpha(Q\|P)$. This proves the claim.
\end{proof}

\begin{proof}[Proof of Theorem \ref{thm:consistency}]
Compactness of $\Theta_k$ and continuity of $\phi_k$ in $\theta$ implies   $\widehat{R}_\alpha^{n,k}(Q\|P)$ are real-valued and measurable. For  $k\in\mathbb{Z}^+$ define
\begin{align}
    R_\alpha^k(Q\|P)\equiv \sup_{\theta\in \Theta_k}\left\{\frac{1}{\alpha-1}\log\int e^{(\alpha-1)\phi_{k,\theta}}dQ-\frac{1}{\alpha}\log\int e^{\alpha \phi_{k,\theta}}dP\right\}\,.
\end{align}
By using the bound
\begin{align}
&\left|R_\alpha^k(Q\|P)-\widehat{R}_\alpha^{n,k}(Q\|P)\right|\\
\leq&\frac{1}{|\alpha-1|}\sup_{\theta\in \Theta_k}\left|\log\left[\int e^{(\alpha-1)\phi_{k,\theta}}dQ\right]-\log\left[\frac{1}{n}\sum_{i=1}^n e^{(\alpha-1)\phi_k(X_i,\theta)}\right]\right|\notag\\
&+\frac{1}{|\alpha|}\sup_{\theta\in \Theta_k} \left|\log\left[\int e^{\alpha \phi_{k,\theta}}dP\right]-\log\left[\frac{1}{n}\sum_{i=1}^n e^{\alpha \phi_k(Y_i,\theta)}\right]\right|\,,\notag
\end{align}
together with the facts that
\begin{align}
\int e^{(\alpha-1)\phi_{k,\theta}}dQ\geq \int e^{-|\alpha-1|\psi_k}dQ\,,\,\,\,\,\int e^{\alpha\phi_{k,\theta}}dP\geq \int e^{-|\alpha|\psi_k}dP,\,\,\,\theta\in\Theta_k\,,
\end{align}
and $\log$ is $1/c$-Lipschitz on $[c,\infty)$ for all $c>0$, we can compute the following for all  $\eta>0$:
\begin{align}\label{eq:prob_bound}
&\left\{\left|R_\alpha^k(Q\|P)-\widehat{R}_\alpha^{n,k}(Q\|P)\right|\geq\eta\right\}\\
\subset&\left\{\sup_{\theta\in \Theta_k}\left|\log\left[\int e^{(\alpha-1)\phi_{k,\theta}}dQ\right]-\log\left[\frac{1}{n}\sum_{i=1}^n e^{(\alpha-1)\phi_k(X_i,\theta)}\right]\right|\geq|\alpha-1|\eta/2 \right.\notag\\
&\left.\hspace{1cm} \text{ and }\sup_{\theta\in\Theta_k}\left|\frac{1}{n}\sum_{i=1}^n e^{(\alpha-1)\phi_k(X_i,\theta)}-\int e^{(\alpha-1)\phi_{k,\theta}}dQ\right|\leq E_Q[e^{-|\alpha-1|\psi_k}]/2\right\}\notag\\
&\cup\left\{\sup_{\theta\in\Theta_k}\left|\frac{1}{n}\sum_{i=1}^n e^{(\alpha-1)\phi_k(X_i,\theta)}-\int e^{(\alpha-1)\phi_{k,\theta}}dQ\right|> E_Q[e^{-|\alpha-1|\psi_k}]/2\right\}\notag\\
&\cup\left\{\sup_{\theta\in \Theta_k} \left|\log\left[\int e^{\alpha \phi_{k,\theta}}dP\right]-\log\left[\frac{1}{n}\sum_{i=1}^n e^{\alpha \phi_k(Y_i,\theta)}\right]\right|\geq |\alpha|\eta/2\right.\notag\\
&\left.\hspace{1cm}\text{ and }\sup_{\theta\in\Theta_k}\left|\frac{1}{n}\sum_{i=1}^n e^{\alpha \phi_k(Y_i,\theta)}-\int e^{\alpha \phi_{k,\theta}}dP \right|\leq E_P[e^{-|\alpha|\psi_k}]/2 \right\}\notag\\
&\cup\left\{\sup_{\theta\in\Theta_k}\left|\frac{1}{n}\sum_{i=1}^n e^{\alpha \phi_k(Y_i,\theta)}-\int e^{\alpha \phi_{k,\theta}}dP \right|> E_P[e^{-|\alpha|\psi_k}]/2\right\}\notag\\
\subset&\left\{\sup_{\theta\in \Theta_k}\left|\frac{1}{n}\sum_{i=1}^n (e^{(\alpha-1)\phi_k(X_i,\theta)}-E_{\mathbb{P}}[e^{(\alpha-1)\phi_k(X_i,\theta)}])\right|\geq \epsilon_1\right\}\notag\\
&\cup\left\{\sup_{\theta\in \Theta_k} \left|\frac{1}{n}\sum_{i=1}^n( e^{\alpha \phi_k(Y_i,\theta)}- E_{\mathbb{P}}[e^{\alpha\phi_k(Y_i,\theta)}])\right|\geq\epsilon_2\right\}\,,\notag\\
\epsilon_1&\equiv \min\{|\alpha-1|\eta E_Q[e^{-|\alpha-1|\psi_k}]/4, E_Q[e^{-|\alpha-1|\psi_k}]/2\}\,,\notag\\
\epsilon_2&\equiv \min\{ |\alpha|\eta E_P[e^{-|\alpha|\psi_k}]/4 , E_P[e^{-|\alpha|\psi_k}]/2\}\,.\notag
\end{align}
For all $\theta\in\Theta_k$ we have $|e^{(\alpha-1)\phi_k(X_i,\theta)}|\leq e^{|\alpha-1|\psi_k(X_i)}\in L^1(\mathbb{P})$ and $|e^{\alpha\phi_k(Y_i,\theta)}|\leq e^{|\alpha|\psi_k(Y_i)}\in L^1(\mathbb{P})$, therefore the uniform law of large numbers (see Lemma 3.10 in \cite{geer2000empirical}) implies convergence in probability:
\begin{align}
&\lim_{n\to\infty}\mathbb{P}\left(\sup_{\theta\in\Theta_k}\left|n^{-1}\sum_{i=1}^n\left(e^{(\alpha-1)\phi_k(X_i,\theta)}-E_{\mathbb{P}}\left[e^{(\alpha-1)\phi_k(X_i,\theta)}\right]\right)\right|\geq\epsilon\right)=0\,,\\
&\lim_{n\to\infty}\mathbb{P}\left(\sup_{\theta\in\Theta_k}\left|n^{-1}\sum_{i=1}^n\left(e^{\alpha\phi_k(Y_i,\theta)}-E_{\mathbb{P}}\left[e^{\alpha\phi_k(Y_i,\theta)}\right]\right)\right|\geq\epsilon\right)=0\notag
\end{align}
for all $\epsilon>0$.  Combined with \req{eq:prob_bound} this implies
\begin{align}\label{eq:Rk_hatRk_lim}
&\lim_{n\to\infty}\mathbb{P}\left(\left|R_\alpha^k(Q\|P)-\widehat{R}_\alpha^{n,k}(Q\|P)\right|\geq\eta\right)=0\,.
\end{align}
 To finish, consider the following two cases.
 \begin{enumerate}
     \item  $R_\alpha(Q\|P)<\infty$: Fix $\delta>0$. The assumption \eqref{eq:k_lim_assump}  implies that there exists $K$ such that for $k\geq K$ we have $R_\alpha(Q\|P)-\delta/2\leq R_\alpha^k(Q\|P)\leq R_\alpha(Q\|P)$. Hence, for $k\geq K$, \req{eq:Rk_hatRk_lim} implies
\begin{align}
&\mathbb{P}(|R_\alpha(Q\|P)-\widehat{R}_\alpha^{n,k}(Q\|P)|\geq \delta)\leq \mathbb{P}(|R_\alpha^k(Q\|P)- \widehat{R}_\alpha^{n,k}(Q\|P)|\geq \delta/2)\to 0
\end{align}
as $n\to\infty$. This proves the claimed result  when $R_\alpha(Q\|P)<\infty$.

\item $R_\alpha(Q\|P)=\infty$: Fix  $M>0$ and $\delta>0$. The assumption \eqref{eq:k_lim_assump}  implies that  there exists $K$ such that for all  $k\geq K$ we have
\begin{align}
R_\alpha^k(Q\|P)\equiv\sup_{\theta\in \Theta_k}\left\{\frac{1}{\alpha-1}\log\int e^{(\alpha-1)\phi_{k,\theta}}dQ-\frac{1}{\alpha}\log\int e^{\alpha \phi_{k,\theta}}dP\right\}\geq M+\delta\,.
\end{align}
Hence for $k\geq K$ we can use \eqref{eq:Rk_hatRk_lim} to obtain
\begin{align}
&\mathbb{P}(\widehat{R}_\alpha^{n,k}(Q\|P)\leq M)\leq
\mathbb{P}\left(|R_\alpha^k(Q\|P)-\widehat{R}_\alpha^{n,k}(Q\|P)|\geq \delta\right)\to 0
\end{align}
  as $n\to\infty$. This proves the claimed result when $R_\alpha(Q\|P)=\infty$.
\end{enumerate}
\end{proof}

\subsection{Applying Theorem \ref{thm:consistency} to Several Classes of Neural Networks}\label{sec:NN_cons_proofs}

Here we prove consistency of the neural network estimators that were discussed in Section  \ref{sec:NN_estimation}. Specifically, we show  they satisfy all of the properties required to apply Theorem \ref{thm:consistency}.
\begin{enumerate}
    \item Measures with compact support: Let $\Omega\subset\mathbb{R}^m$ be compact, $\Phi$ be a family of neural networks that satisfy the universal approximation property \eqref{eq:univ_approx_K}, and let $\Phi_k\subset\Phi$ be the set of  networks with  depth and width bounded by $k$ and parameter values restricted to  $[-a_k,a_k]$, where $a_k\nearrow\infty$. Let $\Psi$ be the set of positive constants. Then  the assumptions of Theorem \ref{thm:consistency} are satisfied and hence the estimator \eqref{eq:renyi_est_def} is consistent.  
    
    \begin{proof}
        To see this, first note that compactness of $\Omega$ implies that every $\phi\in \Phi$ is bounded and so property 1 of Definition \ref{def:Psi_bounded_approx} is trivial.  Property 2 of Definition \ref{def:Psi_bounded_approx} easily follows from the universal approximation property \eqref{eq:univ_approx_K} applied to the compact set $\Omega$.  Therefore $\Phi$ has the  $\Psi$-bounded $L^\infty$ approximation property.  Assumptions 1b and 1c of Lemma \ref{lemma:Renyi_var_Phi} are trivial, as $\psi\in\Psi$ are bounded, and so we have \eqref{eq:var_formula_Phi}.  \req{eq:k_lim_assump} then follows from the fact that $\Phi_k$ increase to $\Phi$.  The remaining items 2 - 5 in Assumption \ref{assump:Phi_k}  then  follow from compactness of $K$ and boundedness of $\psi\in\Psi$.
            \end{proof}

    \item Non-compact support, bounded Lipschitz activation functions:    Let $\Omega=\mathbb{R}^m$ and $\Phi$ be the family of neural networks with 2 hidden layers, arbitrary width, and activation function $\sigma:\mathbb{R}\to\mathbb{R}$. Let $\Phi_k\subset\Phi$ be the set of width-$k$ networks with parameter values restricted to $[-a_k,a_k]$, where $a_k\nearrow\infty$, and let $\Psi$ be the set of positive constants.     If the activation function, $\sigma$,  is bounded and there exists  $(c,d)\subset\mathbb{R}$ on which $\sigma$ is one-to-one and Lipschitz  then  the estimator \eqref{eq:renyi_est_def} is consistent.

    \begin{proof}
        To prove this, first note that boundedness of $\sigma$ implies boundedness of every $\phi\in \Phi$. Therefore 1 of Definition \ref{def:Psi_bounded_approx} holds.  For any $g\in\Lip_b(\mathbb{R}^m)$ we can find $b\in\mathbb{R}$, $a> 0$ such that the range of $(g-b)/a$ is contained in $(c,d)$, and hence  $\sigma^{-1}((g-b)/a)$ is well-defined and continuous. Let $L$ be the Lipschitz constant for $\sigma$ on $(c,d)$  and define $\psi\equiv \max\{|b|+a\|\sigma\|_\infty,\|g\|_\infty\}\in\Psi$. Then $|g|\leq\psi$ and, by the universal approximation property in \cite{pinkus_1999}, for any compact $K\subset\mathbb{R}^m$ and any $\epsilon>0$ there exists a network with one hidden layer, $\phi_\epsilon$, that satisfies
    \begin{align}
    \sup_{K}|  \sigma^{-1}((g-b)/a)-\phi_\epsilon|\leq\epsilon/(aL)\,.  
    \end{align}
Therefore
       \begin{align}
       \sup_{K}|g-(b+a\sigma (\phi_\epsilon))|
       \leq aL\sup_K|\sigma^{-1}((g-b)/a)-\phi_\epsilon|
       \leq  \epsilon\,.
    \end{align}
        We have $b+a\sigma(\phi_\epsilon)\in\Phi$ (as we have simply added a second hidden layer to the network $\phi_\epsilon$) and $|b+a\sigma(\phi_\epsilon)|\leq\psi$. This completes the proof of the $\Psi$-bounded $L^\infty$ approximation property.  Properties 1b and 1c of Lemma  \ref{lemma:Renyi_var_Phi} are trivial and so we can conclude \eqref{eq:var_formula_Phi}.   The sets $\Phi_k$ increase to $\Phi$ and so, combined with \eqref{eq:var_formula_Phi}, we can conclude \eqref{eq:k_lim_assump}. The remaining items in Assumption \ref{assump:Phi_k}  hold due to boundedness of $\psi\in\Psi$, uniform boundedness of the parameter values for $\phi\in\Phi_k$, and boundedness of the activation function.
    \end{proof}

    \item Non-compact support, unbounded Lipschitz activation function: Let $p\in(1,\infty)$ and $\Omega=\mathbb{R}^m$, equipped with the $\ell^p$-norm.  Let $Q$ and $P$ be probability measures on $\Omega$ that have finite moment generating functions everywhere and have densities $dQ/dx$ and $dP/dx$ that are  bounded on compact sets. Define  $\Phi$ be the family of neural networks obtained by using either the ReLU activation function or the GroupSort activation with group size 2 (see \cite{pmlr-v97-anil19a}). Let $\Phi_k\subset\Phi$ be the set of networks with depth and width bounded by $k$ (for ReLU, one can alternatively use networks with depth equal to $3$) and with parameter values restricted to  $[-a_k,a_k]$, where $a_k\nearrow\infty$. Finally, let $\Psi=\{x\mapsto a\|x\|+b:a,b\geq 0\}$, where $\|\cdot\|$ denotes the $\ell^p$-norm.    Then the assumptions of Theorem \ref{thm:consistency} are satisfied, and hence the estimator \eqref{eq:renyi_est_def} is consistent. 
    \begin{proof}
\begin{enumerate}
\item First consider the case of ReLU activation functions. We will show that $\Phi$ has the $\Psi$-bounded $L^\infty$ approximation property (Definition \ref{def:Psi_bounded_approx}). First, all $\phi\in\Phi$ are Lipschitz, hence are bounded by $x\mapsto a\|x\|+b$ for some $a,b\geq 0$. Next, fix $g\in\Lip_b(\Omega)$.  By the universal approximation property \eqref{eq:univ_approx_K} (see \cite{Cybenko1989,pinkus_1999}), for all compact $K\subset\Omega$ and all $\epsilon>0$ there exists a neural network, $\phi_{\epsilon,K}$, with one hidden layer and ReLU activation that satisfies
\begin{align}
\sup_{x\in K}|g(x)-\phi_{\epsilon,K}(x)|<\min\{\epsilon,1\}\,.
\end{align}
Define $\psi=\|g\|_\infty+1\in\Psi$ and
\begin{align}\label{eq:phi_tilde_ReLU}
\widetilde\phi_{\epsilon,K}=\ReLU(-\ReLU(\|g\|_\infty+1-\phi_{\epsilon,K})+2(\|g\|_\infty+1))-(\|g\|_\infty+1)\,.
\end{align}
Note that $\widetilde\phi_{\epsilon,K}\in \Phi$ has depth equal to $3$.  We have  $|g|\leq\psi$, $|\widetilde\phi_{\epsilon,K}|\leq \psi$, and 
\begin{align}
\sup_{x\in K}|g(x)-\widetilde\phi_{\epsilon,K}(x)|=\sup_{x\in K}|g(x)-\phi_{\epsilon,K}(x)|<\epsilon\,.
\end{align}
This proves the $\Psi$-bounded $L^\infty$ approximation property. Properties 1b - 1c of Lemma \ref{lemma:Renyi_var_Phi} follow from the assumption that $Q$ and $P$ have finite moment generating functions everywhere. Therefore we  conclude \eqref{eq:var_formula_Phi}.  Items 1 and 2 in Assumption \ref{assump:Phi_k}  follows from the definition of $\Phi_k$, as in the previous cases. Item 3 follows from the fact that the activation  is Lipschitz and the network parameters, depth, and width of $\phi\in\Phi_k$ are uniformly bounded. Finally, 4 and 5 are implied by  1b and 1c from Lemma \ref{lemma:Renyi_var_Phi}, which were shown above. 

\item        Finally, we  consider the GroupSort case. We start by showing the $\Psi$-bounded $L^p(\mathcal{Q})$ approximation property (Definition \ref{def:Psi_bounded_L1_approx}), where $\mathcal{Q}=\{Q,P\}$. Item 1 of Definition \ref{def:Psi_bounded_L1_approx} follows from the fact that every $\phi\in\Phi$ is $L$-Lipschitz for some $L\geq 0$ and hence   $|\phi(x)|\leq L\|x\|+|\phi(0)|\in\Psi$.  Let $g\in \Lip_b(\Omega)$ with Lipschitz constant $L$ and define $\psi(x)=L\|x\|+\|g\|_\infty+L+1$. Then $\psi\in\Psi$, $|g|\leq \psi$, and, using Theorem 3 in \cite{pmlr-v97-anil19a}, we see that for any compact $K\subset\mathbb{R}^m$ and any $j\in\mathbb{Z}^+$ there exists $\phi_j\in\Phi$ that is $L$-Lipschitz and satisfies
    \begin{align}
        \left(\int_{K\cup \overline{B_1(0)}} |\phi_j-g|^pdx\right)^{1/p}\leq 1/j\,,
        \end{align}
i.e., $\phi_j$ converges to $g$ in $L^p(K\cup \overline{B_1(0)},dx)$ ($\overline{B_1(0)}$ denotes the closed ball of $\ell^p$-radius $1$ centered at $0$). Take a subsequence $\phi_{j_i}$ that converges to $g$ a.e. on $K\cup \overline{B_1(0)}$. In particular, there exists $x_0$ with $\|x_0\|\leq 1$ and $\phi_{j_i}(x_0)\to g(x_0)$.    Hence for $x\in\mathbb{R}^m$ we have
\begin{align}
    |\phi_{j_i}(x)|\leq& |\phi_{j_i}(x)-\phi_{j_i}(x_0)|+|\phi_{j_i}(x_0)-g(x_0)|+|g(x_0)|\\
    \leq& L\|x\| +L+ \|g\|_\infty +|\phi_{j_i}(x_0)-g(x_0)|\,.\notag
    \end{align}
    Therefore, if $\epsilon>0$ then for all $i$ sufficiently large we have $|\phi_{j_i}|\leq \psi$ and
    \begin{align}
        &\sup_{\mu\in \mathcal{Q}}\left(\int_K|g-\phi_{j_i}|^pd\mu\right)^{1/p}\\
        \leq& \max\{\sup_K|dQ/dx|,\sup_K|dP/dx|\}^{1/p}\left(\int_K|g-\phi_{j_i}|^pdx\right)^{1/p}<\epsilon\,.\notag
    \end{align}
    This proves the  $\Psi$-bounded $L^p(\mathcal{Q})$-approximation property. Properties 2b - 2c of Lemma \ref{lemma:Renyi_var_Phi} follow from the assumption that $Q$ and $P$ have finite moment generating functions everywhere. Therefore we conclude \eqref{eq:var_formula_Phi}. Items 1 and 2 in Assumption \ref{assump:Phi_k}  follows from the definition of $\Phi_k$, as in the previous cases. Item 3 follows from the fact that the activation function is Lipschitz and the network parameters, depth, and width of $\phi\in\Phi_k$ are uniformly bounded. Finally, 4 and 5 are implied by  2b and 2c from Lemma \ref{lemma:Renyi_var_Phi}, which were shown above. 
\end{enumerate}
    \end{proof}
\end{enumerate}
\begin{remark}
For ReLU activation, our proof shows that neural networks with 3 hidden layers are sufficient to obtain consistency of the estimator; see \req{eq:phi_tilde_ReLU}.  To the best of the authors' knowledge, it is an open question as to whether the $\Psi$-bounded $L^\infty$ approximation property (or the  $\Psi$-bounded $L^p(\mathcal{Q})$ approximation property) holds for  networks with only one or two hidden layers. If so then consistency for one or two layer networks would follow by the same argument as above. The numerical results in Section \ref{sec:numerical} do suggest that shallow networks yield consistent estimators.
\end{remark}
\begin{remark}
In the case of the GroupSort activation, it was crucial that the variational formula from Theorem \ref{thm:gen_DV} holds when the optimization is performed over the space $\Gamma=\Lip_b(\Omega)$. The required uniform  bounds on the sequence of approximating functions by a fixed $\psi\in\Psi$ would not hold without the Lipschitz restriction.
\end{remark}

\subsection{Complexity Proof}\label{sec:complexity}
Here we will derive   the complexity result \eqref{eq:n_lower_bound} for R{\'e}nyi divergence estimation; we use the same notation as in Theorem \ref{thm:consistency}.  Specifically, we derive finite sample bounds on how well the estimator $ \widehat{R}_\alpha^{n,k}(Q\|P)$ (see \req{eq:estimator_k}) approximates
\begin{align}
R^k_\alpha(Q\|P)\equiv\sup_{\theta\in\Theta_k}\left\{\frac{1}{\alpha-1}\log\left[\int  e^{(\alpha-1)\phi_{k,\theta}}dQ\right]-\frac{1}{\alpha}\log\left[\int e^{\alpha \phi_{k,\theta}}dP\right]\right\}\,.
\end{align}
One should compare this with the corresponding result for KL divergence estimation, Theorem 6 in \cite{MINE_paper}, which has the same qualitative behavior in $\epsilon$, $\delta$, and $d_k$.

\begin{theorem}\label{thm:complexity}
Let $\alpha\in \mathbb{R}\setminus\{0,1\}$,  $\Theta_k\subset\mathbb{R}^{d_k}\cap\{\theta:\|\theta\|\leq K_k\}$, and suppose $\phi_{k,\theta}$ is bounded by $M_k$ and is $L_k$-Lipschitz in $\theta\in\Theta_k$.  Then for all $\epsilon>0$, $\delta>0$ we have
\begin{align}
\mathbb{P}\left(|\widehat{R}_\alpha^{n,k}(Q\|P)-R^k_\alpha(Q\|P)|>\epsilon\right)\leq \delta
\end{align}
whenever
\begin{align}\label{eq:n_complexity_bound}
n\geq \frac{32 D_{\alpha,k}^2}{\epsilon^2}\left(d_k\log(16L_kK_k\sqrt{d_k}/\epsilon)+2d_kM_k\max\{|\alpha|,|\alpha-1|\}+\log(4/\delta)\right)\,,
\end{align}
where $D_{\alpha,k}\equiv\max\{e^{2|\alpha|M_k}/|\alpha|,e^{2|\alpha-1|M_k}/|\alpha-1|\}$. 
\end{theorem}
\begin{proof}
First note that
\begin{align}\label{eq:R_hat_bound1}
&|\widehat{R}_\alpha^{n,k}(Q\|P)-R^k_\alpha(Q\|P)|\\
\leq&\sup_{\theta\in\Theta_k}\left|\frac{1}{\alpha-1}\log\left[\int  e^{(\alpha-1)\phi_{k,\theta}}dQ\right]-\frac{1}{\alpha-1}\log\left[\int  e^{(\alpha-1)\phi_{k,\theta}}dQ_n\right]\right|\notag\\
&+\sup_{\theta\in\Theta_k}\left|\frac{1}{\alpha}\log\left[\int  e^{\alpha\phi_{k,\theta}}dP\right]-\frac{1}{\alpha}\log\left[\int  e^{\alpha\phi_{k,\theta}}dP_n\right]\right|\,.\notag
\end{align}
Given $\eta>0$, take an open cover $B_\eta(\theta_j)$ of $\Theta_k$. It is known that the minimal covering number satisfies \cite{shalev2014understanding}
\begin{align}
N_\eta(\Theta_k)\leq \left(\frac{2K_k\sqrt{d_k}}{\eta}\right)^{d_k}\,.
\end{align}
We let $\eta=\frac{\epsilon}{8L_k}e^{-2M_k\max\{|\alpha|,|\alpha-1|\}}$. For $\theta\in B_\eta(\theta_j)$ we can use the triangle inequality, the uniform bound on $\phi_{k,\theta}$, and the Lipschitz bounds on $\log$, $\exp$, and $\theta\mapsto\phi_{k,\theta}$ to compute
\begin{align}
&\left|\frac{1}{\alpha-1}\log\left[\int  e^{(\alpha-1)\phi_{k,\theta}}dQ\right]-\frac{1}{\alpha-1}\log\left[\int  e^{(\alpha-1)\phi_{k,\theta}}dQ_n\right]\right|\\
\leq&\left|\frac{1}{\alpha-1}\log\left[\int  e^{(\alpha-1)\phi_{k,\theta}}dQ\right]-\frac{1}{\alpha-1}\log\left[\int  e^{(\alpha-1)\phi_{k,\theta_j}}dQ\right]\right|\notag\\
&+\left|\frac{1}{\alpha-1}\log\left[\int  e^{(\alpha-1)\phi_{k,\theta_j}}dQ\right]-\frac{1}{\alpha-1}\log\left[\int  e^{(\alpha-1)\phi_{k,\theta_j}}dQ_n\right]\right|\notag\\
&+\left|\frac{1}{\alpha-1}\log\left[\int  e^{(\alpha-1)\phi_{k,\theta_j}}dQ_n\right]-\frac{1}{\alpha-1}\log\left[\int  e^{(\alpha-1)\phi_{k,\theta}}dQ_n\right]\right|\notag\\
\leq&2L_ke^{2|\alpha-1|M_k}\eta+\left|\frac{1}{\alpha-1}\log\left[\int  e^{(\alpha-1)\phi_{k,\theta_j}}dQ\right]-\frac{1}{\alpha-1}\log\left[\int  e^{(\alpha-1)\phi_{k,\theta_j}}dQ_n\right]\right|\notag\\
\leq&\frac{\epsilon}{4}+\left|\frac{1}{\alpha-1}\log\left[\int  e^{(\alpha-1)\phi_{k,\theta_j}}dQ\right]-\frac{1}{\alpha-1}\log\left[\int  e^{(\alpha-1)\phi_{k,\theta_j}}dQ_n\right]\right|\,.\notag
\end{align}
A similar bound applies to the second term on the right hand side of \eqref{eq:R_hat_bound1}.   Using a union bound and a Lipschitz bound we have
\begin{align}
&\mathbb{P}\left(\max_{j=1,...,N_\eta(\Theta_k)}\left|\frac{1}{\alpha-1}\log\left[\int e^{(\alpha-1)\phi_{k,\theta_j}}dQ\right]-\frac{1}{\alpha-1}\log\left[\int e^{(\alpha-1)\phi_{k,\theta_j}}dQ_n\right]\right|>\epsilon/4\right)\\
\leq&\sum_j \mathbb{P}\left(\frac{1}{|\alpha-1|} e^{|\alpha-1|M_k}\left|\int e^{(\alpha-1)\phi_{k,\theta_j}}dQ-\int e^{(\alpha-1)\phi_{k,\theta_j}} dQ_n\right|>\epsilon/4\right)\,.
\end{align}
Using Hoeffding's inequality we can bound
\begin{align}
\mathbb{P}\left(\left|\int e^{(\alpha-1)\phi_{k,\theta_j}}dQ-\int e^{(\alpha-1)\phi_{k,\theta_j}} dQ_n\right|>c\right)\leq 2\exp\left(-\frac{n c^2}{2\exp(2|\alpha-1|M_k)}\right)
\end{align}
for all $c>0$ and all $j$, hence
\begin{align}
&\mathbb{P}\left(\max_{j=1,...,N_\eta(\Theta_k)}\left|\frac{1}{\alpha-1}\log\left[\int e^{(\alpha-1)\phi_{k,\theta_j}}dQ\right]-\frac{1}{\alpha-1}\log\left[\int e^{(\alpha-1)\phi_{k,\theta_j}}dQ_n\right]\right|>\epsilon/4\right)\\
\leq & 2N_\eta(\Theta_k)\exp\left(-\frac{n\epsilon^2}{32 D_{\alpha,k}^2}\right)\,.\notag
\end{align}
Similarly, we have
\begin{align}
&\mathbb{P}\left(\max_{j=1,...,N_\eta(\Theta_k)}\left|\frac{1}{\alpha}\log\left[\int e^{\alpha\phi_{k,\theta_j}}dP\right]-\frac{1}{\alpha}\log\left[\int e^{\alpha\phi_{k,\theta_j}}dP_n\right]\right|>\epsilon/4\right)\\
\leq & 2N_\eta(\Theta_k)\exp\left(-\frac{n\epsilon^2}{32 D_{\alpha,k}^2}\right)\,.\notag
\end{align}
Combining these we can compute
\begin{align}
&\mathbb{P}\left(|\widehat{R}_\alpha^{n,k}(Q\|P)-R^k_\alpha(Q\|P)|>\epsilon\right)\\
\leq&\mathbb{P}\left(\sup_{\theta\in\Theta_k}\left|\frac{1}{\alpha-1}\log\left[\int  e^{(\alpha-1)\phi_{k,\theta}}dQ\right]-\frac{1}{\alpha-1}\log\left[\int  e^{(\alpha-1)\phi_{k,\theta}}dQ_n\right]\right|>\epsilon/2\right)\notag\\
&+\mathbb{P}\left(\sup_{\theta\in\Theta_k}\left|\frac{1}{\alpha}\log\left[\int  e^{\alpha\phi_{k,\theta}}dP\right]-\frac{1}{\alpha}\log\left[\int  e^{\alpha\phi_{k,\theta}}dP_n\right]\right|>\epsilon/2\right)\notag\\
\leq&\mathbb{P}\left(\max_j\left|\frac{1}{\alpha-1}\log\left[\int  e^{(\alpha-1)\phi_{k,\theta_j}}dQ\right]-\frac{1}{\alpha-1}\log\left[\int  e^{(\alpha-1)\phi_{k,\theta_j}}dQ_n\right]\right|>\epsilon/4\right)\notag\\
&+\mathbb{P}\left(\max_j\left|\frac{1}{\alpha}\log\left[\int  e^{\alpha\phi_{k,\theta_j}}dP\right]-\frac{1}{\alpha}\log\left[\int  e^{\alpha\phi_{k,\theta_j}}dP_n\right]\right|>\epsilon/4\right)\notag\\
\leq&4 N_\eta(\Theta_k)\exp\left(-\frac{n\epsilon^2}{32D_{\alpha,k}^2}\right)\notag\\
\leq&4 \left(\frac{2K_k\sqrt{d_k}}{\eta}\right)^{d_k}\exp\left(-\frac{n\epsilon^2}{32D_{\alpha,k}^2}\right)\,.\notag
\end{align}
Finally, it is straightforward to show that
\begin{align}
4 \left(\frac{2K_k\sqrt{d_k}}{\eta}\right)^{d_k}\exp\left(-\frac{n\epsilon^2}{32D_{\alpha,k}^2}\right)\leq \delta
\end{align}
whenever $n$ satisfies \eqref{eq:n_complexity_bound}.
\end{proof}
\begin{remark}
Though the general techniques for proving Theorem \ref{thm:complexity} are the same as those used in \cite{MINE_paper} to study KL divergence estimators, there are some technical errors in \cite{MINE_paper} that we have corrected in the above derivation.  They have minimal impact on the qualitative behavior, with the exception of the behavior in the bound $M_k$; the correct behavior of the prefactor is exponential in $M_k$ (in \cite{MINE_paper} it was stated to be $M_k^2$).  This impacts both the KL and R{\'e}nyi results. Specifically, the use of Hoeffding's inequality to obtain Eq. (49) in \cite{MINE_paper} must employ a bound on $e^{T_\theta}$ instead of a bound on $T_\theta$, hence the right hand side of that bound should read $2N_\eta(\Theta)\exp(-\frac{\epsilon^2n}{32\exp(2M)})$ (in their notation, there is no subscript $k$ on $M$, $\Theta$, etc.).  This in turn implies that the KL complexity result,  Eq. (45) in \cite{MINE_paper}, should also have a factor of $e^{2M}$ in place of $M^2$. Similar exponential behavior is also present in the result for R{\'e}nyi divergences \eqref{eq:n_complexity_bound}.
\end{remark}

\section*{Acknowledgments}
The research of J.B., M.K., and L. R.-B.  was partially supported by NSF TRIPODS  CISE-1934846.  
The research of M. K. and L. R.-B.   was partially supported by the National Science Foundation (NSF) under the grant DMS-2008970
 and by the Air Force Office of Scientific Research (AFOSR) under the grant FA-9550-18-1-0214. 
The research of P.D. was supported in part by the National Science Foundation (NSF) under the grant DMS-1904992 and by the Air Force Office of Scientific Research (AFOSR) under the grant FA-9550-18-1-0214.
The research of J.W. was partially supported by the Defense Advanced Research Projects Agency (DARPA) EQUiPS program under the grant W911NF1520122.

\bibliography{RenyiVariationalFormulaArxiv.bbl}

\end{document}